\PassOptionsToPackage{authoryear}{natbib}
\documentclass[11pt, a4paper]{article}

\usepackage[utf8]{inputenc}
\usepackage[T1]{fontenc}
\usepackage{geometry}
\usepackage{amsmath, amssymb, amsthm, amsfonts}
\usepackage{mathtools}
\usepackage{physics}   
\usepackage{esint}
\usepackage{graphicx}
\usepackage{hyperref}
\usepackage{xcolor}
\usepackage{enumitem}
\usepackage{dsfont}
\usepackage[square, numbers, sort, compress]{natbib}
\newcommand{\R}{\mathbb{R}}
\newcommand{\E}{\mathbb{E}}

\DeclareMathOperator*{\argmin}{arg\,min}

\theoremstyle{plain}
\newtheorem{theorem}{Theorem}[section]
\newtheorem{lemma}[theorem]{Lemma}
\newtheorem{proposition}[theorem]{Proposition}
\newtheorem{corollary}[theorem]{Corollary}

\theoremstyle{definition}
\newtheorem{definition}{Definition}[section]
\newtheorem{assumption}{Assumption}

\theoremstyle{remark}
\newtheorem{remark}{Remark}[section]

\title{\textbf{Adversarial Robustness of NTK Neural Networks}}
\author{
Yuxuan Hou\textsuperscript{1} \quad
Jingfu Peng\textsuperscript{2} \quad
Yuhong Yang\textsuperscript{2,*}
\\[4pt]
\small \textsuperscript{1}\textit{Qiuzhen College, Tsinghua University}
\\
\small \textsuperscript{2}\textit{Yau Mathematical Sciences Center, Tsinghua University}
\\[3pt]
\small \textsuperscript{*}Corresponding author: \texttt{yyangsc@tsinghua.edu.cn}
}
\date{2026}

\begin{document}

\maketitle

\begin{abstract}
    Deep learning models are widely deployed in safety-critical domains, but remain vulnerable to adversarial attacks. In this paper, we study the adversarial robustness of NTK neural networks in the context of nonparametric regression. We establish minimax optimal rates for adversarial regression in Sobolev spaces and then show that NTK neural networks, trained via gradient flow with early stopping, can achieve this optimal rate. However, in the overfitting regime, we prove that the minimum norm interpolant is vulnerable to adversarial perturbations.
    
    \vspace{0.5em}
\end{abstract}
\section{Introduction}

\subsection{Background}
Deep learning models have achieved remarkable success in various fields, including large language modeling (\cite{brown2020language}), protein structure prediction (\cite{jumper2021highly}), and high-resolution image generation (\cite{rombach2022high}), to name a few. Neural networks are at the core of modern machine learning, and researchers have devoted significant effort to understand their underlying principles, including algorithms (see, e.g.,~\cite{rumelhart1986learning, lecun1998gradient, he2016deep, vaswani2017attention, kingma2015adam, glorot2010understanding}), approximation ability (see, e.g., \cite{schmidthieber2019deep,bos2022convergence,barron2002universal,barron1994approximation}), and generalization ability (see, e.g., \cite{Kohler2021,SchmidtHieber2020,barron1994approximation}). Among all these, the Neural Tangent Kernel (NTK) framework (\cite{jacot2018neural}) has emerged as one of the most successful approaches for understanding overparameterized neural networks by providing an analytically tractable and simplified approximation of their training dynamics. Consequently, the generalization ability of such NTK networks has been widely explored (see, e.g., \cite{li2024eigenvalue,lai2023generalization,chen2024impacts}).

Modern deep learning architectures are typically highly overparameterized and even perfectly interpolating the training data, yet they generalize well to unseen test data---a phenomenon termed benign overfitting (\cite{bartlett2020benign}). As demonstrated by \cite{zhang2017understanding}, deep networks can fit data exactly while still achieving low test error on real datasets, challenging the classical bias--variance trade-off. \cite{belkin2019reconciling} characterizes the double descent risk curve where test error decreases again after the interpolation threshold. These works seem to establish that extreme overparameterization may enable the surprising generalization performance observed in modern deep learning.

However, while most studies focus on the generalization ability of neural networks under standard losses (e.g. $L^2$), a key issue remains: Overparameterized and overfitted neural networks are often vulnerable to adversarial attacks. As highlighted by foundational works (\cite{szegedy2014intriguing, goodfellow2015explaining}), these subtle,  imperceptible perturbations to input data ($X \to X'$) can drastically degrade the performance of the trained model. From the viewpoint of \cite{ilyas2019adversarial}, the vulnerability is inherently tied to networks' tendency to memorize non-robust features.

There is little theoretical understanding of the adversarial robustness of neural networks. This gap persists even within the highly tractable NTK regime. Some studies focus on adversarial training for NTK networks to minimize adversarial loss. For example, \cite{gao2019convergence} prove that the optimal robust parameters in a sufficiently small neighborhood of the initialized parameters can be attained with proper training time. In contrast, \cite{fu2025theoretical} show that prolonged adversarial training eventually degrades adversarial robustness. Despite these insights, analyses of standard-trained NTK models remain scarce.

Furthermore, existing literature on the adversarial robustness of overfitted estimators focuses almost exclusively on high-dimensional linear regression. \cite{donhauser2021interpolation} show that overfitted regression estimators are adversarially vulnerable even when there is no noise under a high-dimensional setting. \cite{hao2024surprising} demonstrate that adversarial risk can diverge due to overfitting. However, the analysis of \cite{hao2024surprising} relies on stringent assumptions of the true function that effectively reduce the NTK to a linear model. Currently, the adversarial robustness of standard-trained and overfitted NTK models remains largely unexplored.

In the above background, we aim to answer two fundamental questions regarding standard-trained neural networks in the deep NTK regime: First, does the standard training procedure inherently yield adversarially robust estimators? Second, does overfitting exacerbate adversarial vulnerability, or can it somehow escape from it?

\subsection{Contributions}
In this work, we provide a theoretical analysis of the adversarial robustness of NTK neural networks through the lens of nonparametric regression. Unlike classical nonparametric settings that often rely on H\"older spaces, we specifically analyze this problem in Sobolev spaces, as they naturally coincide with the Reproducing Kernel Hilbert Space (RKHS) of the NTK (which will be stated in Section~4). Our main contributions are summarized as follows:

\begin{itemize}
    \item \textbf{Minimax Optimal Rates for Adversarial Regression:} We establish the minimax learning rate under adversarial perturbations in Sobolev spaces. Furthermore, we design an algorithm based on Lepski's method that is adaptive to an unknown Sobolev index.

    \item \textbf{Adversarial Optimality of NTK Neural Networks with Early Stopping:} We first build the equivalence between the RKHS of the NTK kernel and a Sobolev space on $[0,1]^d$, then prove that NTK neural networks trained via gradient flow and early stopping can achieve the aforementioned minimax optimal adversarial risk.

    \item \textbf{Overfitting Harms Adversarial Robustness:} We not only show that the adversarial risk of the overfitted estimator diverges, but also precisely characterize when this occurs during the training process. Furthermore, we demonstrate that achieving standard $L^2$ consistency via exact data interpolation does not guarantee robustness, even in fixed dimensions. Specifically, we prove that the spiky kernel gradient flow estimator in \cite{haas2023mind} suffers from infinite adversarial risk after a finite training time, despite successfully achieving consistency.
\end{itemize}

\subsection{Other Related Works}
\textbf{Minimax Rates in Adversarial Nonparametric Regression.} \cite{zhao2024robust} and \cite{moradi2025adversarial} examine settings where the input training data is corrupted. \cite{peng2024minimax} and \cite{peng2024adversarial} establish the first nonparametric minimax rates of convergence for the H\"older class under future $X$ attacks. Building upon this framework, \cite{peng2026damage} further show that data interpolation schemes inherently degrade adversarial robustness for learning H\"older smooth functions.

\textbf{Adversarial Learning.} There are many methods to enhance adversarial robustness in practice. For example, empirical defenses predominantly rely on adversarial training via Projected Gradient Descent (PGD) (\cite{madry2018towards}), which formulates the training process as a min-max robust optimization problem, or employ regularization frameworks like TRADES (\cite{zhang2019theoretically}) to explicitly balance standard accuracy and adversarial risk. However, since empirical defenses are frequently circumvented by stronger adaptive attacks, the community is profoundly concerned with certified robustness. In this domain, Randomized Smoothing (\cite{rekavandi2024certified}) achieves certified robustness by convolving the base model's predictions with Gaussian noise (a concept we adapt for continuous regression in Section 5.3), while IBP (\cite{gowal2018effectiveness}) offers a scalable certified-training alternative by combining interval bounds with a tailored loss and an \(\epsilon\)-curriculum.

\section{Problem Setup}

\subsection{Adversarial Risk}
We consider the nonparametric regression model with true function $f^*$ and noise $\xi$, which consists of $n$ i.i.d. samples $(X_i,Y_i)_{i=1}^n$:
\begin{align*}
    Y = f^*(X) + \xi, \quad i=1,\dots,n,
\end{align*}
where $f^*$ is assumed to be in a function class $\mathcal{M}$.

We aim to construct an estimator which has good performance under the following adversarial risk.
\begin{definition}[Adversarial Risk]
    Consider an adversary capable of perturbing the input $X$ within an $r$-neighborhood $B(X,r) \triangleq \{x' \in [0,1]^d : \|x' - X\| \le r\}$ where $\|\cdot \|$ is the Euclidean norm. We define the adversarial risk of an estimator $\hat{f}$ relative to the true function $f^*$ as:
    \begin{align*}
        R_A(\hat{f}, f^*) \triangleq \mathbb{E}_{\mathcal{D}_n, X} \left[ \sup_{x' \in B(X,r) \cap [0,1]^d} |\hat{f}(x') - f^*(X)|^2 \right],
    \end{align*}
    where the expectation is taken over both the training data $\mathcal{D}_n = \{(X_i, Y_i)\}_{i=1}^n$ and the test point $X$.
\end{definition}
Note that when $r=0$, the adversarial risk reduces to the standard $L^2$ risk.

\subsection{Structure of Neural Network}
Neural networks can be used as a type of nonparametric estimator of $f^*$.
We consider a standard fully connected feedforward neural network with $L$ hidden layers. Let $x \in \mathbb{R}^d$ be the input value. We denote the width of the $l$-th hidden layer as $m_l$ for $l = 1, \dots, L+1$, and let $m_0 = d,m_{L+1}=1$. The network is parameterized by a collection of weights and biases $\theta = \{W^{(l)}, b^{(l)}\}_{l=0}^{L}$, where $W^{(l)} \in \mathbb{R}^{m_{l+1} \times m_l}$ and $b^{(l)} \in \mathbb{R}^{m_{l+1}}$.

The forward propagation of the standard neural network is defined recursively. Let the initial input be $\alpha^{(0)}(x) = x$. For each hidden layer $l = 1, \dots, L$, the pre-activation $z^{(l)}(x)$ and post-activation $\alpha^{(l)}(x)$ are computed as:
\begin{align*}
    z^{(l)}(x) &= W^{(l-1)} \alpha^{(l-1)}(x) + b^{(l-1)}, \\
    \alpha^{(l)}(x) &= \sigma\left(z^{(l)}(x)\right),
\end{align*}
where $\sigma(\cdot)$ is a non-linear activation function, which we specify as the ReLU function $\sigma(u) = \max(u, 0)$ applied element-wise. The final scalar output of the network $f^{NN}(x; \theta): \mathbb{R}^d \to \mathbb{R}$ is given by a linear transformation of the features from the last hidden layer:
\begin{align*}
    f^{NN}(x; \theta) = W^{(L)} \alpha^{(L)}(x) + b^{(L)}.
\end{align*}
We can approximate the true function $f^*$ by learning the parameters of $f^{NN}$.

To theoretically analyze these dynamics, we study an idealized NTK estimator; its detailed construction is in Section 4. In the overparameterized regime, this NTK construction serves as a simplification of the standard neural network defined above.

\subsection{Sobolev Spaces}
To study the performance of the NTK estimator in an adversarial setting, we assume $f^*$ is in a Sobolev space, which is closely related to the NTK kernel (for details see Section 4).
Let $\mathcal{F}$ and $\mathcal{F}^{-1}$ denote the Fourier transform and its inverse operator, respectively. For a function $f \in L^1(\mathbb{R}^d)$, the Fourier transform $\mathcal{F}\{f\}$ is defined by the integral
\begin{align*}
    \mathcal{F}(f)(\xi) = \frac{1}{(2\pi)^{d/2}} \int_{\mathbb{R}^d} f(x) e^{-i \langle x, \xi \rangle} \, dx, \quad \xi \in \mathbb{R}^d,
\end{align*}
where $\langle x, \xi \rangle = \sum_{k=1}^d x_k \xi_k$ denotes the standard Euclidean inner product on $\mathbb{R}^d$.
By a density argument, the Fourier transform operator can be uniquely extended to a unitary operator on $L^2(\R^d)$.

\begin{definition}[Sobolev Space $H^s$]
    Let $s \ge 0$ and $d \ge 1$. The Sobolev space $H^s(\R^d)$ is defined via the Fourier transform as:
    \begin{align*}
        H^s(\R^d) := \left\{ f \in L^2(\R^d) : \|f\|_{H^s(\R^d)}^2 := \int_{\R^d} (1+\abs{\xi}^2)^s \abs{\mathcal{F}({f})(\xi)}^2 \, d\xi < \infty \right\}.
    \end{align*}
\end{definition}
\begin{remark}
    For any non-negative integer $s$, the Sobolev space $H^s(\mathbb{R}^d)$ consists of all functions in $L^2(\mathbb{R}^d)$ whose weak partial derivatives up to order $s$ exist and belong to $L^2(\mathbb{R}^d)$.
\end{remark}
For a bounded domain $\Omega$, we define the Sobolev space $H^s(\Omega)$ as the restriction of functions in $H^s(\R^d)$ to $\Omega$. The space $H^s(\Omega)$ is equipped with the standard quotient norm:
\begin{align*}
    \|f\|_{H^s(\Omega)} := \inf \left\{ \|g\|_{H^s(\R^d)} : g \in H^s(\R^d) \text{ and } g|_{\Omega} = f \text{ a.e.} \right\}.
\end{align*}
We define the Sobolev ball on $[0,1]^d$ as $\mathcal{H}^{s}(L) = \{f^* \in H^s([0,1]^d) : \|f^*\|_{H^s([0,1]^d)} \le L\}$.

\subsection{Technical Assumptions}
We consider the nonparametric regression model defined in Section 2.1. Throughout most of this paper, we let $\mathcal{M}=\mathcal{H}^s(L)$ and impose the following assumptions on the covariates and random errors.
\begin{assumption}[Covariates]
    The covariates $X_i \in [0,1]^d$ are independent and identically distributed (i.i.d.) with a probability density function $f_X$ with respect to the Lebesgue measure. We assume the density is bounded away from zero and infinity, i.e., there exist constants $c, C > 0$ such that $c \le f_X(x) \le C$ for all $x \in [0,1]^d$.
\end{assumption}

\begin{assumption}[Noise]
    The noise terms $\xi_i$ are independent and identically distributed (i.i.d.) and independent of the covariates $\{X_i\}_{i=1}^n$. We assume that $\mathbb{E}[\xi_i] = 0$ and $\xi_i$ is $\sigma$-sub-Gaussian, i.e., its moment generating function satisfies $\mathbb{E}[\exp(\lambda \xi_i)] \le \exp(\lambda^2 \sigma^2 / 2)$ for all $\lambda \in \mathbb{R}$.
\end{assumption}

\subsection{Goal of the Paper}
Specifically, we seek to answer the following three questions:
\begin{enumerate}
    \item What is the minimax rate of learning a function in the Sobolev ball $\mathcal{H}^s(L)$ under adversarial risk? That is, we aim to determine $\inf_{\hat{f}} \sup_{f^* \in \mathcal{H}^{s}(L)} R_A(\hat{f}, f^*)$.
    \item Can neural networks operating in the NTK regime achieve this minimax rate?
    \item What happens to the model's adversarial robustness after it overfits (i.e., as the training time $t \to \infty$)?
\end{enumerate}
Those questions are answered in Sections 3, 4, and 5 respectively.
\section{Adversarial Minimax Rate of Sobolev Spaces}
\subsection{Main Minimax Theorem}

We begin by deriving the information-theoretic lower bound, which characterizes the minimum possible adversarial risk achievable by any estimator.

\begin{proposition}[Minimax Lower Bound]
    \label{thm:lower_bound}
    Suppose the regression model in section 2.1 follows the two assumptions about covariate and noise in Section 2.4, and $r\leq 1$. There exists a constant $c=c(s,d,L,\sigma) > 0$ such that the minimax adversarial risk over the Sobolev ball $\mathcal{H}^{s}(L)$ satisfies:
    \begin{align*}
        \inf_{\hat{f}} \sup_{f^* \in \mathcal{H}^{s}(L)} R_A(\hat{f}, f^*) \ge c(r^{2\min(1,s)} + n^{-\frac{2s}{2s+d}}).
    \end{align*}
\end{proposition}

The proof relies on the construction of a least-favorable hypothesis. Detailed derivations are provided in Appendix A.

To achieve this lower bound, we analyze a \textit{Projected and Truncated Kernel Ridge Regression} (PT-KRR) estimator.

It is a classical result that the Sobolev ellipsoid $\mathcal{H}^{s}(L)=\{ f\in H^s([0,1]^d) \mid \|f\|_{H^s}\leq L\}$ is a closed and convex subset of $L^{2}([0,1]^d)$ for any $s>0$. This geometric property allows for a well-defined projection operation.

\textbf{Case 1:} $s > d/2$.
Let $\tilde{f}$ be the standard Kernel Ridge Regression (KRR) estimator:
\begin{align*}
    \tilde{f} = \argmin_{g \in H^{s}([0,1]^d)} \left\{ \frac{1}{n}\sum_{i=1}^{n}(Y_{i}-g(X_{i}))^{2} + \lambda \|g\|_{H^s([0,1]^d)}^{2} \right\},
\end{align*}
where the regularization parameter is chosen as $\lambda \asymp n^{-\frac{2s}{2s+d}}$. While $\tilde{f}$ is optimal for standard $L^2$ loss, it can be modified to make it more robust. We define our robust estimator $\hat{f}$ as the $L_2$-projection of $\tilde{f}$ onto the Sobolev ball:
\begin{align*}
    \hat{f} = \mathcal{P}_{\mathcal{H}^{s}(L)}(\tilde{f}) \triangleq \argmin_{g \in \mathcal{H}^{s}(L)} \|g - \tilde{f}\|_{L_{2}([0,1]^d)}.
\end{align*}
\textbf{Case 2:} $s \leq d/2$.
In this low-smoothness regime, we elevate the penalty to a higher-order space. Take $H=H^d([0,1]^d)$, and let $\tilde{f}$ be the regularized KRR estimator:
\begin{align*}
    \tilde{f} = \argmin_{g \in H^{d}([0,1]^d)} \left\{ \frac{1}{n}\sum_{i=1}^{n}(Y_{i}-g(X_{i}))^{2} + \lambda \|g\|_{H^d([0,1]^d)}^{2} \right\},
\end{align*}
where the regularization parameter $\lambda$ is calibrated as $\lambda \asymp n^{-\frac{2d}{2s+d}}$. We first project $\tilde{f}$ onto the target Sobolev ball in the $L_2$ norm to obtain the intermediate estimator $\bar{f}$:
\begin{align*}
    \bar{f} = \mathcal{P}_{\mathcal{H}^{s}(L)}(\tilde{f}) \triangleq \argmin_{g \in \mathcal{H}^{s}(L)} \|g - \tilde{f}\|_{L_{2}}.
\end{align*}
We now apply a hard frequency truncation to define the final estimator $\hat{f}$:
Let $\chi \in C_c^\infty(\mathbb{R}^d)$ be a smooth radial cut-off function such that $0 \le \chi \le 1$, with
\[
\chi(x) =\chi(|x|)
\begin{cases}
1 & \text{if } |x| \le \frac{1}{2}, \\
0 & \text{if } |x| \ge 1.
\end{cases}
\]
Such a function can be constructed explicitly: we begin with the standard smooth non-negative function $f: \mathbb{R} \to \mathbb{R}$ defined by
\[
f(t) =
\begin{cases}
\exp\left(-\frac{1}{t}\right), & t > 0, \\
0, & t \le 0.
\end{cases}
\]
Next, we define a smooth transition function $g: \mathbb{R} \to [0,1]$ by
\[
g(t) = \frac{f(t)}{f(t) + f(1-t)}.
\]
We define $\chi: \mathbb{R}^d \to [0,1]$ explicitly as
\[
\chi(x)= 1 - g\left( \frac{4|x|^2 - 1}{3} \right),
\]
and $\chi_a(x)=\chi(ax)$.
By the Sobolev Extension Theorem (see, e.g., \cite{grafakos2014modern}), we can extend $\bar{f}$ to an element of $H^s(\mathbb{R}^d)$.
Now we construct our final estimator
\begin{align*}
    \hat{f} = \mathcal{F}^{-1} \left( \mathcal{F}(\bar{f}) \chi_r \right).
\end{align*}
Taking a restriction on $[0,1]^d$ gives the final estimator.
\begin{proposition}[Minimax Upper Bound]
    \label{thm:upper_bound}
    Under the two assumptions in section 2.4, the estimator $\hat{f}$ achieves the minimax optimal rate. That is, for a constant $C=C(s,d,L,\sigma) > 0$:
    \begin{align*}
        \sup_{f^* \in \mathcal{H}^{s}(L)} R_A(\hat{f}, f^*) \le C(r^{2\min(1,s)} + n^{-\frac{2s}{2s+d}}).
    \end{align*}
\end{proposition}
The proof utilizes the properties of Sobolev spaces. The detailed analysis is presented in Appendix A.

Combining the lower and upper bounds, we characterize the exact minimax rate for adversarial nonparametric regression in Sobolev spaces.

\begin{theorem}[Minimax Rate under Adversarial Risk]
    The optimal minimax rate for regression in $H^s([0,1]^d)$ with adversarial perturbation radius $r$ satisfies:
    \begin{align*}
        \inf_{\hat{f}} \sup_{f^* \in \mathcal{H}^{s}(L)} R_A(\hat{f}, f^*)  \asymp
            r^{2\min(1,s)} + n^{-\frac{2s}{2s+d}}.
    \end{align*}
\end{theorem}
\begin{remark}
Our result is similar to \cite{peng2024adversarial}, but we focus on different function classes. Specifically, the Sobolev space is a broader class in the sense that there exists a continuous embedding $C^\beta(\Omega) \hookrightarrow H^{s-\epsilon}(\Omega)$ for any $\epsilon>0$. Furthermore, Sobolev spaces are closely related to kernel methods in modern machine learning, whereas Hölder classes primarily appear in nonparametric estimation.
\end{remark}
\subsection{Adaptivity via Lepski's Method}

We first observe that achieving simultaneous adaptivity to the adversarial radius $r$ is fundamentally impossible when $s \le d/2$, which can be formulated in the following theorem.
\begin{theorem}
There exists a function $f_{ill}\in H^{\frac{d}{2}}([0,1]^d)$ such that there does not exist an estimator $\hat{f}: \mathcal{D}_n \to C([0,1]^d)$ and a constant $M=M(s,d,L,\sigma)$ such that for all n and r:
\begin{align*}
    R_A(\hat{f},f_{ill})\leq M(n^{\frac{-2s}{2s+d}}+r^{2\min(1,s)}).
\end{align*}
\end{theorem}

The proof utilizes the unboundedness of $f_{ill}$; see Appendix A.

Assume the true smoothness index $s$ lies in a known interval $[s_{\min}, s_{\max}]$. Without loss of generality, we assume $s_{\max} > d/2$. We propose the following adaptive procedure, which extends the misspecified kernel ridge regression (KRR) framework introduced by \cite{zhang2024optimality}. Let the Hilbert space be $\mathcal{H} = H^{s_{\max}}$.
Define the kernel ridge regression estimator as
\begin{align*}
    \hat{f_v}=\argmin_{f\in \mathcal{H}}\frac{1}{n}\sum_{i=1}^n (y_i-f(x_i))^2+\frac{1}{v}\|f\|^2_\mathcal{H}.
\end{align*}
We construct a discrete parameter grid $\mathcal{V} = \{v_1, \dots, v_K\}$, where $v_i = 2^{i-1} n^{\frac{2s_{\max}}{2s_{\max}+d}}$ and the total number of grids is $K = \left\lceil \log_2 \left( \frac{n^{\frac{2s_{\max}}{2s_{\min}+d}}}{n^{\frac{2s_{\max}}{2s_{\max}+d}}} \right) \right\rceil + 1$. The optimal parameter $v^*$ is selected via Lepski's balancing principle:
\begin{align*}
    v^* = \min \left\{ v_i \in \mathcal{V} \;\middle|\; \|\hat{f}_{v_i} - \hat{f}_{v_j}\|_{L^2([0,1]^d)} \leq \log^2(n) \frac{v_j^{\frac{d}{4s_{\max}}}}{\sqrt{n}}, \quad \forall j > i \right\}.
\end{align*}
To ensure adversarial robustness, we apply a frequency truncation to $\hat{f}_{v^*}$. Let $\mathcal{F}$ denote the Fourier transform operator. The final estimator $\hat{f}$ is defined as:
\begin{align*}
    \hat{f}_0 = \mathcal{F}^{-1} \left( \mathcal{F}(\hat{f}_{v^*}) \chi_r \right).
\end{align*}
Finally, we enforce a boundedness condition:
\[
\hat{f} =
\begin{cases}
\hat{f}_0, & \text{if } \int_{[0,1]^d} \sup_{x' \in B(x,r)} |\hat{f}_0(x') - \hat{f}_0(x)|^2 \, \mathrm{d}x + \|\hat{f}_0\|_{L^2([0,1]^d)} \leq \log(n), \\
0, & \text{otherwise}.
\end{cases}
\]

\begin{theorem}
\label{thm:final_adv_risk}
The resulting estimator $\hat{f}$ achieves the adaptive adversarial risk bound (optimal up to a log factor):
\begin{align*}
    \sup_{f\in H^s(L)}R_A(\hat{f}, f^*) \lesssim \left( r^{2\min(1,s)} + n^{-\frac{2s}{2s+d}} \right) (\log^5 n).
\end{align*}
\end{theorem}

\section{NTK Neural Networks Achieve Adversarial Optimality}
\subsection{Construction of NTK Estimator}

We adopt the NTK neural network framework analyzed in Li et al.\ (2024), considering a fully connected ReLU neural network with $L$ hidden layers, where $L \ge 2$. Suppose we receive independent and identically distributed training data $(x_i,y_i)$ for $1\leq i \leq n$, where $x_i\in \mathbb{R}^d$ and $y_i\in \mathbb{R}$. Our goal is to train an estimator $\hat{f}:\mathbb{R}^d\to \mathbb{R}$.

Let $m_1, m_2, \dots, m_L$ denote the widths of the $L$ hidden layers, and let $m_{L+1}=1$ be the dimension of the output layer. To govern the asymptotic behavior of the network, we define $m = \min(m_1, m_2, \dots, m_L)$ as the minimum hidden layer width and assume that the maximum width satisfies $\max(m_1, m_2, \dots, m_L) \le C_{\text{width}}m$ for some absolute constant $C_{\text{width}}$.

Let $x \in \mathbb{R}^d$ be the input. The network consists of two parallel sub-networks indexed by parity $p \in \{1, 2\}$, defined recursively as:
\begin{align*}
    \alpha^{(1,p)}(x) &= \sqrt{\frac{2}{m_1}}\sigma\left(A^{(p)}x + b^{(0,p)}\right) \in \mathbb{R}^{m_1}, \\
    \alpha^{(l,p)}(x) &= \sqrt{\frac{2}{m_l}}\sigma\left(W^{(l-1,p)}\alpha^{(l-1,p)}(x)\right) \in \mathbb{R}^{m_l}, \quad \text{for } l=2, \dots, L,
\end{align*}
where $A^{(p)} \in \mathbb{R}^{m_1 \times d}$ and $W^{(l-1,p)} \in \mathbb{R}^{m_l \times m_{l-1}}$ are the weight matrices, $b^{(0,p)} \in \mathbb{R}^{m_1}$ is the bias vector, and $\sigma(u) = \max(u, 0)$ denotes the ReLU activation function applied element-wise.

The scalar output of each sub-network is:
\begin{equation*}
    g^{(p)}(x;\theta) = W^{(L,p)}\alpha^{(L,p)}(x) + b^{(L,p)} \in \mathbb{R},
\end{equation*}
where $\theta$ represents the collection of all network parameters flattened into a single column vector. The final output of the neural network, $f(x;\theta)$, is the scaled difference between these sub-networks:
\begin{equation*}
    f(x;\theta) = \frac{\sqrt{2}}{2} \left[ g^{(1)}(x;\theta) - g^{(2)}(x;\theta) \right].
\end{equation*}

The parameters are initialized using a ``mirror initialization'' strategy: the parameters for the $p=1$ sub-network are drawn i.i.d.\ from a standard normal distribution $\mathcal{N}(0,1)$, and the parameters for the $p=2$ sub-network are set to perfectly match their $p=1$ counterparts. This structural symmetry ensures that the initial output is identically zero, $f(x;\theta_0) \equiv 0$.

We train the network using continuous-time gradient descent (gradient flow) to minimize the empirical squared loss $\mathcal{L}(\theta) = \frac{1}{2n}\sum_{i=1}^n(f(x_i;\theta) - y_i)^2$. Denoting $X = (x_1, \dots, x_n)$ and $Y = (y_1, \dots, y_n)^\top$, the parameter vector evolves according to the ordinary differential equation:
\begin{equation*}
    \dot{\theta}_t = -\nabla_\theta \mathcal{L}(\theta_t) = -\frac{1}{n}\nabla_\theta f(X;\theta_t)(f(X;\theta_t)-Y).
\end{equation*}

Let $\hat{f}_t^{NN}(x) := f(x;\theta_t)$ be the network predictor at training time $t$. Applying the chain rule, the evolution of the function itself is governed by the following gradient flow equation:
\begin{equation*}
    \frac{d}{dt}\hat{f}_t^{NN}(x) = -\frac{1}{n} K_t(x, X) (\hat{f}_t^{NN}(X) - Y),
\end{equation*}
where $K_t(x, x') = \langle \nabla_\theta f(x;\theta_t), \nabla_\theta f(x';\theta_t) \rangle$ is the time-varying, empirical NTK.

As the minimum width $m \to \infty$ (the over-parameterized or ``lazy training'' regime), the empirical kernel $K_t(x, x')$ concentrates tightly around a deterministic, time-invariant NTK, $K^{NT}(x, x')$. To explicitly formulate $K^{NT}$, we define the arc-cosine kernels $\kappa_0, \kappa_1 : [-1, 1] \to \mathbb{R}$ associated with the ReLU activation:
\begin{align*}
    \kappa_0(u) &= \frac{1}{\pi}(\pi - \arccos u), \\
    \kappa_1(u) &= \frac{1}{\pi}\left[\sqrt{1-u^2} + u(\pi - \arccos u)\right].
\end{align*}

Let $\tilde{x} = (x, 1) \in \mathbb{R}^{d+1}$ denote the input augmented with a bias term, and let $\bar{u} = \frac{\langle \tilde{x}, \tilde{x}' \rangle}{\|\tilde{x}\| \|\tilde{x}'\|}$ be the cosine similarity. The deterministic limiting NTK, $K^{NT}: \mathbb{R}^d \times \mathbb{R}^d \to \mathbb{R}$, is given explicitly by:
\begin{equation*}
    K^{NT}(x, x') = \|\tilde{x}\| \|\tilde{x}'\| \left( \sum_{h=0}^{L} \kappa_1^{(h)}(\bar{u}) \prod_{k=h}^{L-1} \kappa_0\left(\kappa_1^{(k)}(\bar{u})\right) \right) + 1,
\end{equation*}
where $\kappa_1^{(h)}$ denotes the $h$-fold composition of $\kappa_1$, with $\kappa_1^{(0)}(u) = u$.

Replacing the empirical kernel with its deterministic limit yields the idealized kernel regression predictor $\hat{f}_t^{NTK}(x)$. Because $\hat{f}_0^{NTK}(x) \equiv 0$, solving its gradient flow analytically gives:
\begin{equation*}
    \hat{f}_t^{NTK}(x) = K_x(X)^\top (K(X, X))^{-1} \left( I - e^{-\frac{t}{n} K(X, X)} \right) Y,
\end{equation*}
where $K_x(X) = [K^{NT}(x, x_1), \dots, K^{NT}(x, x_n)]^\top$ and $K(X, X) \in \mathbb{R}^{n \times n}$ is the kernel Gram matrix evaluated on the training data.

Crucially, provided the width $m$ is sufficiently large (polynomial in $n$, $L$, and $\delta^{-1}$), the neural network stays within this kernel regime throughout training. For any $\delta \in (0,1)$, with probability at least $1-\delta$ over the random initialization:
\begin{equation*}
    \sup_{t \ge 0} \sup_{x \in \Omega} \left| \hat{f}_t^{NTK}(x) - \hat{f}_t^{NN}(x) \right| \le O_m\left(R^2 m^{-\frac{1}{12}} \sqrt{\ln m}\right),
\end{equation*}
where $\Omega \subset B(0, R)$ is the bounded domain containing the data. This uniform convergence allows the generalization properties of the over-parameterized neural network to be completely characterized by the functional analytic properties of the associated NTK.

In the above NTK regime, \cite{li2024eigenvalue} showed that under a certain source condition, the early-stopping NTK estimator is minimax optimal with respect to the $L^2$ loss.

The connection between the NTK and our functional analysis framework can be established by using results from \cite{chen2024impacts}:
\begin{lemma}[RKHS of ReLU NTK]
    \label{lemma:ntk_rkhs}
    The Reproducing Kernel Hilbert Space (RKHS) associated with $K^{NT}$ on the domain $[0,1]^d$ is norm-equivalent to the Sobolev space $H^{s}([0,1]^d)$ with smoothness index $s = \frac{d+1}{2}$.
\end{lemma}

See Appendix B for the proof.

\subsection{Optimality of Early Stopping}

Leveraging the RKHS equivalence, we now show that the wide neural network, when trained with early stopping, achieves the minimax optimal adversarial rate derived in Section 2.
\begin{theorem}[Adversarial Optimality of Wide Networks]
    Consider a nonparametric regression model in section 2.1 with true function $f^* \in H^s([0,1]^d)$ where $s>\frac{d}{2}$ (the condition is necessary to provide the boundedness of the true function), and $\|f^*\|_{H^s} \leq R$. Let $r>0$ be the adversarial perturbation radius. If a neural network is trained via gradient flow and early-stopped at $t^* = n^{\frac{d+1}{2s+d}}$, the expected adversarial risk of the  NTK estimator satisfies:
    \begin{align*}
        [\mathcal{R}_A(\hat{f}_{t^*}^{NTK}, f^*)] \lesssim n^{-\frac{2s}{2s+d}} + r^{2\min(1,s)}.
    \end{align*}
    Furthermore, for sufficiently wide neural networks, $\hat{f}_{t^*}^{NN}$ achieves the identical rate. This matches the minimax lower bound in Theorem \ref{thm:lower_bound}, establishing that NTK networks are minimax optimal learners under adversarial attacks.
\end{theorem}

The gradient flow up to $t^*$ acts as a spectral filter, analogous to the regularization parameter $\lambda \asymp (t^*)^{-1}$ in kernel ridge regression; the early stopping ensures the expected RKHS norm remains bounded, $\mathbb{E}[\|\hat{f}_{t^*}^{NTK}\|_{H^s}^2] \lesssim 1$. See Appendix B for detailed derivations.

\section{Overfitting Harms Adversarial Robustness}

\subsection{The Vulnerability of the Minimum Norm Interpolant}

In the previous section, we demonstrated that wide neural networks trained with early stopping achieve minimax optimality. A natural question arises: what happens if training continues until convergence?

As $t \to \infty$, the gradient flow solution converges to the \textit{minimum norm interpolant}, denoted by $\hat{f}_{\infty}$. In the context of the NTK regime, this is the function in the RKHS that fits the training data perfectly while minimizing the RKHS norm:
\begin{align*}
    \hat{f}_{\infty} = \argmin_{g \in H^{s}} \|g\|_{H^s} \quad \text{subject to} \quad g(x_i) = y_i, \quad \forall i=1,\dots,n.
\end{align*}
While literature has shown that under certain circumstances $\hat{f}_{\infty}$ is not consistent under $L^2$ loss (\cite{buchholz2022kernel}), we show that the adversarial risk can diverge to $\infty$. It is worth noting that our Theorem 5.1 aligns with the findings of \cite{peng2026damage}, although their results do not directly imply ours.

\begin{theorem}[Adversarial Divergence of Interpolants]
    \label{thm:overfitting_lower_bound}
    Let $r_n$ be a sequence of perturbation radii satisfying the density condition $nr_n^d \to \infty$ as $n \to \infty$. For the nonparametric model in Section 2.1, under the two assumptions stated in Section 2.4, we further assume that $\xi \sim \mathcal{N}(0,\sigma)$ and $\|f^*\|_{L^{\infty}} \leq C$. Then, there exist constants $D=D(C,d,\sigma), M=M(C,d,\sigma)> 0$ such that for all $n>M$, the adversarial risk of the minimum norm interpolant is lower bounded by:
    \begin{align*}
        R_A(\hat{f}_{\infty}, f^*) \ge D \log(nr_n^d).
    \end{align*}
\end{theorem}

Furthermore, we generalize Theorem 2.2(1) in \cite{lai2023generalization} as a core technical tool for establishing the following result. While their result focuses on shallow networks on $\mathbb{R}$, our result is valid for any dimension and number of layers.
\begin{lemma}
    Define the kernel matrix $M = K^{NT}(X,X)$. Then there exists a constant $C$ depending only on $d$ and $L$ such that the smallest eigenvalue of $M$ satisfies
    \[
        \frac{1}{C}\min_{i,j}|x_i-x_j|\leq \lambda_{\min}\leq C\min_{i,j}|x_i-x_j|.
    \]
\end{lemma}

Using the tool, we can establish a stronger version of the theorem:
\begin{theorem}
    With the same assumptions as in Theorem 5.1, there exist constants $D=D(C,d,\sigma), M=M(C,d,\sigma)> 0$ such that for all $n>M$ and $t>n^{\frac{6}{d}+1}$, the adversarial risk of the minimum norm interpolant is lower bounded by:
    \begin{align*}
        R_A(\hat{f}_{t}, f^*) \ge D\sigma^2 \log(nr_n^d).
    \end{align*}
\end{theorem}

See Appendix C for proof.

We also consider the kernel ridgeless estimator. While it is consistent under the $L^2$ error (the ``benign overfitting in fixed dimension'' phenomenon in \cite{haas2023mind}), we now show that its adversarial risk diverges.

\begin{proposition}[adapted from \cite{haas2023mind}]
    Consider the kernel on $[0,1]^d$ defined as
    \[
        k_n(x,y) = K^{\mathrm{NT}}(x,y) + \frac{1}{n^{1/5}}  \exp\left( \frac{-\|x-y\|}{n^{-3/d}} \right).
    \]
    Then, the ridgeless minimal norm interpolator $\hat{f}_n(x) = k_n(x,X)k_n(X,X)^{-1}y$ is consistent for any given $f^* \in C_c^\infty([0,1]^d)$.
\end{proposition}

We now prove that solving the kernel regression via gradient flow harms adversarial robustness due to overfitting.

\begin{theorem}
    With the same assumptions as in Theorem 5.1, the adversarial risk of the estimator
    \[
        \hat{f}_n^t(x) = k_n(x,X) \left( I - \exp\left( -\frac{t}{n}k_n(X,X) \right) \right) k_n(X,X)^{-1}y,
    \]
    satisfies
    \[
        R_A(\hat{f}_n^t,f^*) \gtrsim \log(n r_n^d) \quad \forall t>n^
    {\frac{6}{d}+1}.
    \]
\end{theorem}
\section{Experiments}
\label{sec:experiments}

In this section, we present empirical evaluations that corroborate our theoretical findings regarding the adversarial robustness of wide neural networks. Specifically, we aim to validate the minimax optimality of early stopping (Theorem 4.2) and the vulnerability of the minimum norm interpolant in the overfitting regime (Theorem 5.1 and Theorem 5.3).

\subsection{Adversarial Risk Dynamics in the Exact NTK Regime}
We first evaluate the exact idealized NTK dynamics under gradient flow, i.e., the following expression:
\[
\hat{f}_t^{\mathrm{NTK}}(x) = K_x(X)^\top (K(X, X))^{-1} \left( I - e^{-\frac{t}{n} K(X, X)} \right) Y.
\]

Figure \ref{fig:ntk_adv_risk} illustrates the evolution of adversarial risk $R_A(\hat{f}_t^{\mathrm{NTK}}, f^*)$ across different sample sizes $n$ under three distinct data distributions. These datasets were deliberately chosen to validate the theoretical bounds across varying intrinsic dimensions, target mappings, and noise conditions:

\begin{itemize}
    \item \textbf{1D Synthetic Dataset:} Inputs are sampled uniformly $x_i \sim \mathrm{Unif}(0, 1)$. The response variables are generated by a continuous target function $f^*(x) = \sin(2\pi x)$ corrupted by additive i.i.d. Gaussian noise $y_i = f^*(x_i) + \epsilon_i$, where $\epsilon_i \sim \mathcal{N}(0, \sigma^2)$, $\sigma=0.3$.
    \item \textbf{Real-World Diabetes Dataset:} A standard empirical regression benchmark\footnote{Data retrieved from \url{https://www4.stat.ncsu.edu/~boos/var.select/diabetes.tab.txt}} comprising $d=10$ baseline physiological variables (e.g., age, BMI, blood pressure, and six blood serum measurements) used to predict a quantitative measure of disease progression. The input features are standardized to zero mean and strictly projected onto the unit sphere $\mathbb{S}^{d-1}$.
    \item \textbf{High-Dimensional Synthetic Dataset ($d=5$):} Inputs are drawn uniformly from a hypercube $x_i \sim \mathrm{Unif}([0,1]^d)$. The target function maps the empirical mean of the feature vector to a sinusoidal response, $f^*(x) = \sin\left(2\pi \cdot \frac{1}{d}\sum_{j=1}^d x^{(j)}\right)$ and the output is corrupted by additive i.i.d. Gaussian noise $y_i = f^*(x_i) + \epsilon_i$, where $\epsilon_i \sim \mathcal{N}(0, \sigma^2)$, $\sigma=0.3$.
\end{itemize}

We evaluated the adversarial risk at different $r$, with 0 (standard MSE), 0.02 and 0.05.
Across all dimensions and data distributions, the trajectories universally exhibit a pronounced U-shape against training time $t$, empirically confirming two core theoretical claims:

\begin{itemize}
    \item \textbf{Descent Phase (Optimal Early Stopping):} In the early stages, the adversarial risk $R_A$ rapidly decreases to a global minimum at an optimal stopping time; this descent strictly aligns with the minimax optimal rate derived in Theorem 4.2.
    \item \textbf{Ascent Phase (Benign vs. Adversarial Overfitting):} As $t \to \infty$, the adversarial risk predictably deteriorates. Whether it diverges drastically (as seen in the 1D case) or saturates at a suboptimal high risk (as seen in higher dimensions and real data), it verifies the lower bound established in Theorem 5.3.
\end{itemize}

\begin{figure}[!htbp]
    \centering
    \includegraphics[width=0.9\textwidth]{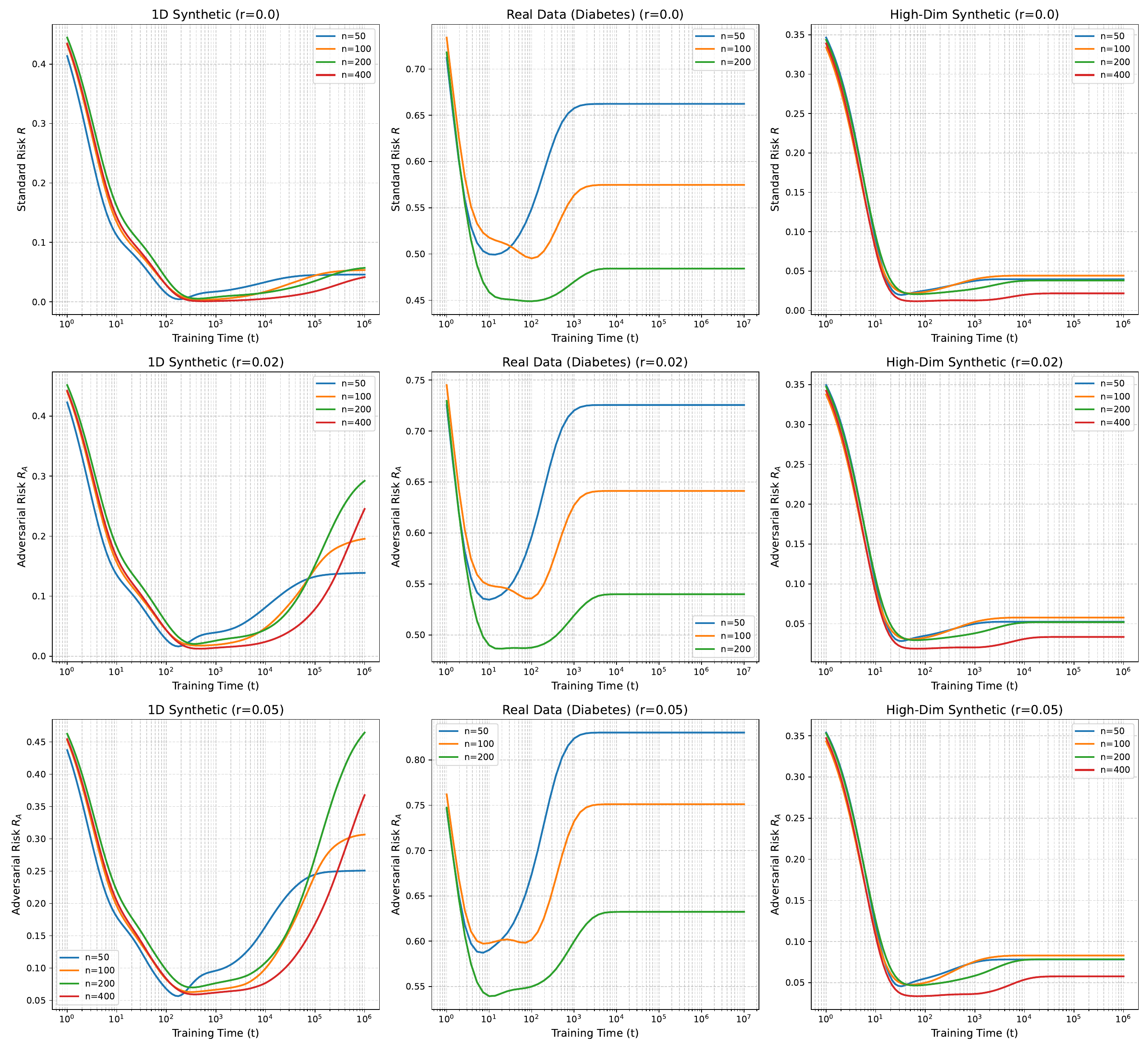}
    \caption{Evolution of Adversarial Risk $R_A$ over training time $t$ in the exact NTK regime. From left to right: 1D Synthetic data with Gaussian noise, real-world Diabetes regression dataset on $\mathbb{S}^{d-1}$, and High-Dim ($d=5$) Synthetic data. The universally consistent U-shaped curves highlight the fundamental necessity of early stopping (or equivalent spectral regularization) to prevent the severe degradation of adversarial risk caused by noise interpolation.}
    \label{fig:ntk_adv_risk}
\end{figure}

\subsection{Training Dynamics and Function Space Spikes of Wide ReLU Networks}
We further train a wide single-hidden-layer ReLU network (width $m = 100,000$) using full-batch Gradient Descent to approximate the continuous-time gradient flow on 1-D data with $n=10$.

As depicted in Figure \ref{fig:ntk_training_dynamics} (Left), the adversarial loss curve is in a U shape, which validates our theoretical result.

Figure \ref{fig:ntk_training_dynamics} (Right) provides an intuitive geometric explanation in the function space. The estimator obtained at the early stopping point smoothly approximates $f^*$, thereby ensuring adversarial stability. Conversely, the final interpolant $\hat{f}_\infty$ exhibits sharp local peaks exactly at the noisy training data points. These spikes are necessary to perfectly interpolate the Gaussian noise $\xi_i$, which intrinsically breaks the local modulus of continuity.

\begin{figure}[!htbp]
    \centering
    \includegraphics[width=0.9\textwidth]{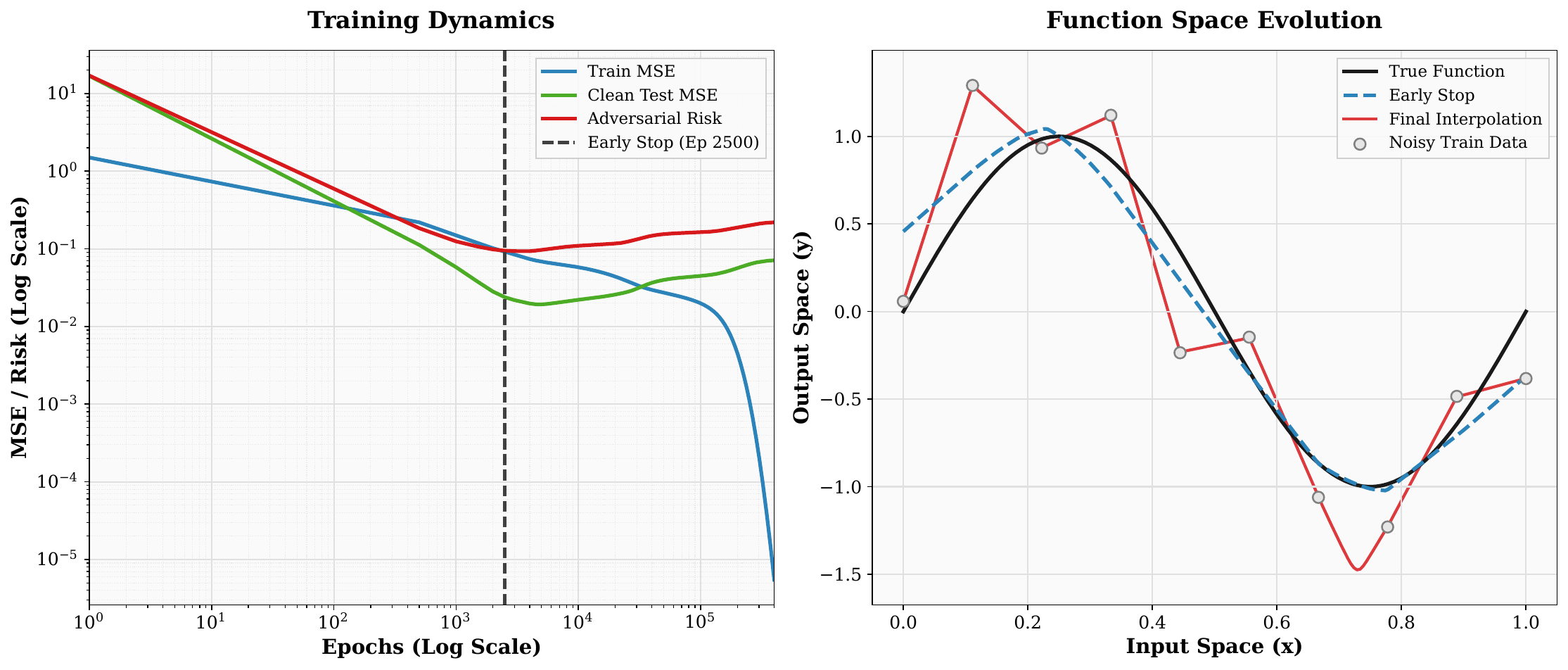}
    \caption{Left: Training dynamics of a wide ReLU network. Right: Function space visualization comparing the smooth early-stopped estimator with the highly oscillatory final minimum norm interpolant.}
    \label{fig:ntk_training_dynamics}
\end{figure}

\subsection{Mitigating Overfitting Vulnerabilities via Randomized Smoothing}

In the baseline exact NTK regime, the adversarial risk trajectory exhibits a severe ascent phase as $t \rightarrow \infty$. To evaluate potential defenses against this interpolation-induced brittleness, we follow the randomized $\alpha$-smoothing method (\cite{rekavandi2024certified}). For a given test point, we sample perturbed inputs from a Gaussian distribution with a standard deviation of $0.1$. We then apply an $\alpha$-trimming filter with a truncation rate of 0.35 or 0.49; also, we provide the NN without smoothing in contrast.

\begin{figure}[!htbp]
    \centering
    \includegraphics[width=0.9\textwidth]{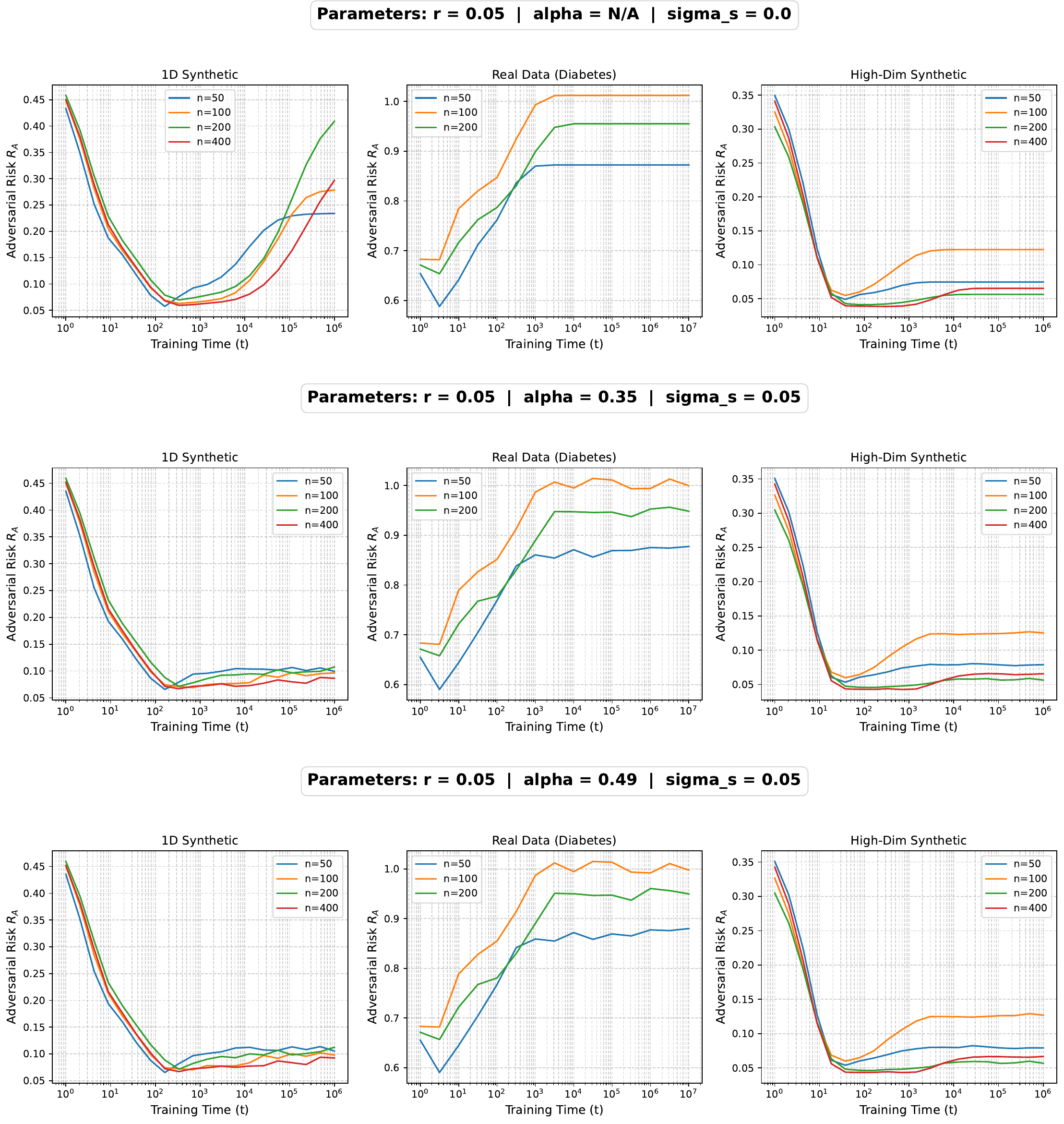}
    \caption{Evolution of Adversarial Risk $R_A$ over training time $t$ with $\alpha$-trimming smoothing.}
    \label{fig:ntk_alpha_trimmed}
\end{figure}

We replicate the idealized NTK dynamics experiment described in Section 5.1 across sample sizes $n \in \{50, 100, 200, 400\}$. Figure \ref{fig:ntk_alpha_trimmed} reveals that while the $\alpha$-smoothed estimator maintains a highly stable risk on 1D synthetic data—even as training time approaches $10^6$ steps, it performs poorly on high-dimensional synthetic and real-world data. The result reveals a limitation of the method, indicating that the smoothing of black-box estimators warrants further investigation.

\clearpage
\section{Conclusion and Future Directions}

In this paper, we have provided a theoretical characterization of the adversarial robustness of wide neural networks in the context of nonparametric regression. Our analysis establishes the minimax optimal rates for adversarial risk in Sobolev spaces $H^s$, and shows that wide neural networks, trained via gradient flow with early stopping, are minimax optimal, effectively balancing the bias-variance trade-off while controlling adversarial instability. In contrast, we proved that the minimum norm interpolant diverges in adversarial risk, highlighting the necessity of early stopping as a form of regularization in safety-critical applications.

Our results rely on the NTK approximation, which holds in the infinite-width limit. An important direction for future work is to extend these results to finite-width networks or to the ``feature learning'' regime, where the tangent kernel evolves during training and may offer adaptivity beyond that of static kernels. In addition, our results are established for a fixed dimension $d$; characterizing the high-dimensional regime in which $d \to \infty$ as $n \to \infty$ remains an open challenge.
\bibliography{references}
\appendix
\section{Proof of Section 3}
\subsection{Upper Bound}
\begin{lemma}
    The Sobolev ball $\mathcal{H}^s(L)$ is a closed and convex subset of $L^2([0,1]^d)$.
\end{lemma}
\begin{proof}
We will prove the convexity and closedness of the Sobolev ball $\mathcal{H}^s(L) = \{ f \in H^s([0,1]^d) : \|f\|_{H^s} \le L \}$ in $L^2([0,1]^d)$ separately.

\textbf{Step 1: Convexity}
Let $f, g \in \mathcal{H}^s(L)$ and $t \in [0, 1]$. We aim to show that the convex combination $h = tf + (1-t)g$ also belongs to $\mathcal{H}^s(L)$. By the absolute homogeneity and the triangle inequality of the norm in the Sobolev space $H^s([0,1]^d)$, we have:
\[
    \|h\|_{H^s} = \|tf + (1-t)g\|_{H^s} \le \|tf\|_{H^s} + \|(1-t)g\|_{H^s} = t\|f\|_{H^s} + (1-t)\|g\|_{H^s}.
\]
Since $f, g \in \mathcal{H}^s(L)$, we know that $\|f\|_{H^s} \le L$ and $\|g\|_{H^s} \le L$. Substituting these bounds yields:
\[
    \|h\|_{H^s} \le tL + (1-t)L = L.
\]
Thus, $tf + (1-t)g \in \mathcal{H}^s(L)$, which proves that $\mathcal{H}^s(L)$ is a convex set.

\textbf{Step 2: Closedness in $L^2([0,1]^d)$}
Since $L^2([0,1]^d)$ is a separable metric space, it is enough to show sequential closedness: for any sequence $\{f_n\}_{n=1}^{\infty} \subset \mathcal{H}^s(L)$ that converges to some $f \in L^2([0,1]^d)$ in the $L^2$ topology (i.e., $\lim_{n \to \infty} \|f_n - f\|_{L^2} = 0$), we must have $f \in \mathcal{H}^s(L)$.

First, since $\{f_n\} \subset \mathcal{H}^s(L)$, the sequence is uniformly bounded in the Hilbert space $H^s([0,1]^d)$, meaning $\|f_n\|_{H^s} \le L$ for all $n$. 
Because $H^s([0,1]^d)$ is a Hilbert space, there exists a subsequence $\{f_{n_k}\}_{k=1}^{\infty}$ that converges weakly to some limit $\tilde{f} \in H^s([0,1]^d)$. We denote this as $f_{n_k} \rightharpoonup \tilde{f}$ in $H^s([0,1]^d)$.

Next, consider the natural inclusion map $\iota: H^s([0,1]^d) \hookrightarrow L^2([0,1]^d)$. Since $s \ge 0$, this map is a bounded linear operator. Because bounded linear operators preserve weak convergence, the subsequence also converges weakly in $L^2([0,1]^d)$:
\[
    \iota(f_{n_k}) \rightharpoonup \iota(\tilde{f}) \implies f_{n_k} \rightharpoonup \tilde{f} \text{ in } L^2([0,1]^d).
\]

On the other hand, we are given that the full sequence $\{f_n\}$ converges strongly to $f$ in $L^2([0,1]^d)$, which implies that its subsequence $\{f_{n_k}\}$ also converges strongly to $f$ in $L^2([0,1]^d)$. Since strong convergence implies weak convergence in normed spaces, we also have $f_{n_k} \rightharpoonup f$ in $L^2([0,1]^d)$.

By the uniqueness of weak limits in $L^2([0,1]^d)$, we deduce that $\tilde{f} = f$ almost everywhere. Therefore, $f$ essentially belongs to $H^s([0,1]^d)$.

Finally, the norm functional $\|\cdot\|_{H^s}$ is weakly lower semicontinuous in the Hilbert space $H^s([0,1]^d)$. Using the weak convergence $f_{n_k} \rightharpoonup f$ in $H^s([0,1]^d)$, we obtain:
\[
    \|f\|_{H^s} \le \liminf_{k \to \infty} \|f_{n_k}\|_{H^s}.
\]
Since $\|f_{n_k}\|_{H^s} \le L$ for all $k$, it follows that:
\[
    \|f\|_{H^s} \le L.
\]
This proves that $f \in \mathcal{H}^s(L)$, concluding the proof that $\mathcal{H}^s(L)$ is a closed subset of $L^2([0,1]^d)$.
\end{proof}
Combining the above fact and Proposition 5.3 in \cite{brezis2011functional}, we immediately establish the following lemma.
\begin{lemma}[Contraction of Projection]
For any $f \in \mathcal{H}^{s}(L)$, the projected estimator satisfies:
\begin{align*}
    \|\hat{f} - f\|_{L^{2}([0,1]^d)} \le \|\tilde{f} - f\|_{L^{2}([0,1]^d)}.
\end{align*}
\end{lemma}

The probabilistic bound in \cite{zhang2024optimality} directly implies the following result by simply taking an integral.
\begin{corollary}[Adapted from \cite{zhang2024optimality}]
\begin{align*}
    \sup_{f \in \mathcal{H}^{s}(L)} \E_{\mathcal{D}_n} \|\tilde{f} - f\|_{L^{2}}^2 \lesssim n^{-\frac{2s}{2s+d}}.
\end{align*}
\end{corollary}

\begin{proof}
Let $X = \|\tilde{f} - f\|_{L^{2}}^2$. Theorem 1 from \cite{zhang2024optimality} establishes that there exists a constant $C_s$ (depending only on the function class $s$) such that for any $\delta \in (0, 1)$ and $n$,
\begin{align*}
    \mathbb{P}\left( X > C_sn^{-\frac{2s}{2s+d}} \log^2\left(\frac{6}{\delta}\right) \right) \leq \delta.
\end{align*}
Let $\epsilon_n^2 = C_s n^{-\frac{2s}{2s+d}}$. For a non-negative random variable $X$, its expectation can be bounded by integrating its tail probability:
\begin{align*}
    \E[X] = \int_{0}^{\infty} \mathbb{P}(X > t) \, dt.
\end{align*}
To apply the high-probability bound, we match the term inside the probability by setting $t = \epsilon_n^2 \log^2\left(\frac{6}{\delta}\right)$. This yields $\delta(t) = 6 \exp\left(-\frac{\sqrt{t}}{\epsilon_n}\right)$. 
Notice that the bound $\mathbb{P}(X > t) \leq \delta(t)$ is valid only when $\delta(t) \in (0, 1)$, which rigorously requires $t > \epsilon_n^2 \log^2(6)$. Therefore, we truncate and split the integral at $t_0 = \epsilon_n^2 \log^2(6)$:
\begin{align*}
    \E[X] = \int_{0}^{t_0} \mathbb{P}(X > t) \, dt + \int_{t_0}^{\infty} \mathbb{P}(X > t) \, dt.
\end{align*}
For the first integral, we apply the trivial upper bound $\mathbb{P}(X > t) \leq 1$:
\begin{align*}
    \int_{0}^{t_0} \mathbb{P}(X > t) \, dt \leq \int_{0}^{t_0} 1 \, dt = t_0 = \epsilon_n^2 \log^2(6).
\end{align*}
For the second integral, since $\delta(t) < 1$, we can strictly apply the bound from Theorem 1:
\begin{align*}
    \int_{t_0}^{\infty} \mathbb{P}(X > t) \, dt \leq \int_{t_0}^{\infty} 6 \exp\left(-\frac{\sqrt{t}}{\epsilon_n}\right) \, dt.
\end{align*}
We evaluate this integral using the substitution $y = \frac{\sqrt{t}}{\epsilon_n}$, which implies $t = \epsilon_n^2 y^2$ and $dt = 2 \epsilon_n^2 y \, dy$. The lower limit of integration changes from $t_0$ to $\log(6)$:
\begin{align*}
    \int_{t_0}^{\infty} 6 \exp\left(-\frac{\sqrt{t}}{\epsilon_n}\right) \, dt &= \int_{\log(6)}^{\infty} 6 e^{-y} \cdot 2 \epsilon_n^2 y \, dy \nonumber \\
    &= 12 \epsilon_n^2 \int_{\log(6)}^{\infty} y e^{-y} \, dy.
\end{align*}
Using integration by parts ($\int y e^{-y} dy = -e^{-y}(y + 1)$), we obtain:
\begin{align*}
    \int_{\log(6)}^{\infty} y e^{-y} \, dy = \left[ -e^{-y}(y + 1) \right]_{\log(6)}^{\infty} = e^{-\log(6)}(\log(6) + 1) = \frac{\log(6) + 1}{6}.
\end{align*}
Substituting this back, the second integral is bounded by:
\begin{align*}
    12 \epsilon_n^2 \cdot \frac{\log(6) + 1}{6} = 2 \epsilon_n^2 (\log(6) + 1).
\end{align*}
Combining both parts, we arrive at the expectation bound:
\begin{align*}
    \E[X] &\leq \epsilon_n^2 \log^2(6) + 2 \epsilon_n^2 (\log(6) + 1) \nonumber \\
    &= \epsilon_n^2 \left( \log^2(6) + 2\log(6) + 2 \right).
\end{align*}
Since $\log^2(6) + 2\log(6) + 2$ is an absolute constant, we conclude that:
\begin{align*}
    \E \|\tilde{f} - f\|_{L^{2}}^2 \leq C_s \left( \log^2(6) + 2\log(6) + 2 \right) n^{-\frac{2s}{2s+d}} \lesssim n^{-\frac{2s}{2s+d}}.
\end{align*}
This completes the proof.
\end{proof}
\subsubsection{Proof of $d=1$, $s>d/2$}

Using the inequality $(a+b)^{2} \le 2a^{2} + 2b^{2}$, we decompose the adversarial risk:
\begin{align*}
    R_{A}(\hat{f}, f) &= \mathbb{E}_{X,\mathcal{D}_n} [\sup_{x^{\prime} \in A(X)} |\hat{f}(x^{\prime}) - \hat{f}(X) + \hat{f}(X) - f(X)|^{2}] \nonumber \\
    &\le 2\mathbb{E}_{X,\mathcal{D}_n}[|\hat{f}(X) - f(X)|^{2}] + 2\mathbb{E}_{X,\mathcal{D}_n} [\sup_{x^{\prime} \in A(X)} |\hat{f}(x^{\prime}) - \hat{f}(X)|^{2}].
\end{align*}
The first term is the standard risk, and the second term represents the adversarial instability.
To bound the second term, we now establish the following tools.
\begin{lemma}[Modulus of Continuity in $H^{s}$]
There is a universal constant $C=C_s$ such that for any function $g \in H^{s}([0,1])$ with $s \in (1/2, 1)$, where $A(X) = [X-r, X+r] \cap [0,1]$ is the local neighborhood, the expected local variation satisfies:
\begin{align*}
    \mathbb{E}_X\left[\sup_{x^{\prime} \in A(X)} |g(x^{\prime}) - g(X)|^{2}\right] \le C r^{2s} \|g\|_{H^{s}}^{2}.
\end{align*}
\end{lemma}

\begin{lemma}[Localized Morrey's Inequality]
Let $J \subset \mathbb{R}$ be an interval of length $|J| = 2r$. For any $g \in H^{s}(J)$ with $s > 1/2$, and for any $u, v \in J$, the following pointwise estimate holds:
\begin{align*}
    |g(u) - g(v)|^{2} \le C_{s} r^{2s-1} \iint_{J \times J} \frac{|g(x) - g(y)|^{2}}{|x-y|^{1+2s}} dx dy.
\end{align*}
\end{lemma}

\begin{proof}
The proof relies on approximating the point value $g(u)$ by a sequence of averages over shrinking intervals. Let $J_{0} = J$. For a fixed point $u \in J$, consider a sequence of nested intervals $J_{k} = J \cap (u - 2^{-k}r, u + 2^{-k}r)$. Note that $|J_{k}| \approx 2^{-k}r$.

Let $\overline{g}_{k}$ denote the integral average of $g$ over $J_{k}$:
\begin{align*}
    \overline{g}_{k} := \frac{1}{|J_{k}|} \int_{J_{k}} g(z) dz.
\end{align*}
By the Lebesgue Differentiation Theorem, assuming $u$ is a Lebesgue point, we have $g(u) = \lim_{k \to \infty} \overline{g}_{k}$. Thus, we can write the telescoping sum:
\begin{align*}
    g(u) - \overline{g}_{0} = \sum_{k=0}^{\infty} (\overline{g}_{k+1} - \overline{g}_{k}).
\end{align*}
Now we estimate the difference between consecutive averages:
\begin{align*}
    |\overline{g}_{k+1} - \overline{g}_{k}| &= \left|\frac{1}{|J_{k+1}|}\int_{J_{k+1}}g(x)dx - \frac{1}{|J_{k}|}\int_{J_{k}}g(y)dy\right| \\
    &\le \frac{1}{|J_{k+1}||J_{k}|} \int_{J_{k+1}}\int_{J_{k}} |g(x) - g(y)| dy dx.
\end{align*}
Since $J_{k+1} \subset J_{k}$, for any $x, y \in J_{k}$, we have $|x-y| \le |J_{k}| \approx 2^{-k}r$. We multiply and divide by $|x-y|^{\frac{1}{2}+s}$. Applying the Cauchy-Schwarz inequality:
\begin{align*}
    |\overline{g}_{k+1} - \overline{g}_{k}| \le \frac{1}{|J_{k+1}||J_{k}|} \left(\iint_{J_{k} \times J_{k}} \frac{|g(x)-g(y)|^{2}}{|x-y|^{1+2s}} dx dy\right)^{1/2} \times \left(\iint_{J_{k+1} \times J_{k}} |x-y|^{1+2s} dx dy\right)^{1/2}.
\end{align*}
The second integral is bounded by $|J_k|^{s+3/2}$. Substituting this back:
\[
    |\overline{g}_{k+1} - \overline{g}_{k}| \le C |J_{k}|^{-2} |J_{k}|^{s+\frac{3}{2}} [g]_{H^{s}(J_{k})} = C |J_{k}|^{s-\frac{1}{2}} [g]_{H^{s}(J_{k})}.
\]
Substituting $|J_{k}| = 2^{-k}r$ and summing the series (which converges since $s > 1/2$):
\begin{align*}
    |g(u) - \overline{g}_{0}| \le \sum_{k=0}^{\infty} |\overline{g}_{k+1} - \overline{g}_{k}| \le C' r^{s-\frac{1}{2}} [g]_{H^{s}(J)}.
\end{align*}
Finally, for any $u, v \in J$, by the triangle inequality and squaring both sides, we obtain:
\begin{align*}
    |g(u) - g(v)|^{2} \le C'' r^{2s-1} \iint_{J \times J} \frac{|g(x)-g(y)|^{2}}{|x-y|^{1+2s}} dx dy.
\end{align*}
\end{proof}

\noindent\textbf{Proof of Lemma A.3.} We integrate the localized Morrey inequality derived above with respect to $X$ over $[0, 1]$. Let $J_x = [X-r, X+r] \cap [0,1]$.
\begin{align*}
    \E_X[\sup_{x' \in J_x} |g(x') - g(X)|^2]
    &\le \int_0^1 \left( C r^{2s-1} \iint_{J_x \times J_x} \frac{|g(u) - g(v)|^2}{|u-v|^{1+2s}} \, du \, dv \right) f_X(x) dx \\
    &\lesssim \int_0^1 \left( C r^{2s-1} \iint_{J_x \times J_x} \frac{|g(u) - g(v)|^2}{|u-v|^{1+2s}} \, du \, dv \right) dx.
\end{align*}
By Fubini's theorem, we exchange the order of integration. The condition $u, v \in J_x$ is equivalent to $x \in [u-r, u+r]$ and $x \in [v-r, v+r]$. The intersection of these intervals has measure at most $2r$. Thus:
\begin{align*}
    \text{RHS} \lesssim C r^{2s-1} \cdot 2r \iint_{[0,1]^2} \frac{|g(u) - g(v)|^2}{|u-v|^{1+2s}} \, du \, dv \leq C' r^{2s} \|g\|_{H^{s}([0,1])}^{2}.
\end{align*}

\begin{lemma}[Modulus of Continuity in $H^{1}$]
There is a universal constant $C=C_s$ such that for any function $g \in H^{s}([0,1])$ with $s \geq 1$, the local variation satisfies:
\begin{align*}
    \mathbb{E}_X\left[\sup_{x^{\prime} \in A(X)} |g(x^{\prime}) - g(X)|^{2}\right] \le C r^{2} \|g\|_{H^{1}}^{2} \le C r^{2} \|g\|_{H^{s}}^{2}.
\end{align*}
\end{lemma}
\begin{proof}
    Once we have proved the lemma for smooth functions $g$, a density argument extends the result to the entire space $H^1$.
    
    First, we establish a pointwise bound. By the Fundamental Theorem of Calculus, for any fixed $x \in [0,1]$ and $x^{\prime} \in A(x)$ (where $|x - x^{\prime}| \le r$), we have:
    \begin{equation*}
        g(x^{\prime}) - g(x) = \int_{x}^{x^{\prime}} g'(t) \, dt.
    \end{equation*}
    Applying the Cauchy-Schwarz inequality yields:
    \begin{equation*}
        |g(x^{\prime}) - g(x)|^{2} = \left| \int_{x}^{x^{\prime}} 1 \cdot g'(t) \, dt \right|^{2} \le \left| \int_{x}^{x^{\prime}} 1^{2} \, dt \right| \cdot \left| \int_{x}^{x^{\prime}} |g'(t)|^{2} \, dt \right|.
    \end{equation*}
    Since $|x - x^{\prime}| \le r$, the first integral is bounded by $r$. The second integral is bounded by the integral over the entire neighborhood $A(x) = [x-r, x+r] \cap [0,1]$. Thus:
    \begin{align*}
        \sup_{x^{\prime} \in A(x)} |g(x^{\prime}) - g(x)|^{2} \le r \int_{A(x)} |g'(t)|^{2} \, dt.
    \end{align*}
    Next, we compute the expectation with respect to $X \sim U([0,1])$. Integrating the inequality over $x \in [0,1]$:
    \begin{equation*}
        \mathbb{E}_X\left[\sup_{x^{\prime} \in A(X)} |g(x^{\prime}) - g(X)|^{2}\right] \lesssim \int_{0}^{1} \left( r \int_{A(x)} |g'(t)|^{2} \, dt \right) dx.
    \end{equation*}
    We rewrite the inner integral using the indicator function $\mathds{1}_{\{t \in A(x)\}} = \mathds{1}_{\{|t-x| \le r\}}$. Since the integrand is non-negative, by Fubini's Theorem, we can interchange the order of integration:
    \begin{align*}
        \int_{0}^{1} \left( \int_{0}^{1} |g'(t)|^{2} \mathds{1}_{\{|t-x| \le r\}} \, dt \right) dx 
        &= \int_{0}^{1} |g'(t)|^{2} \left( \int_{0}^{1} \mathds{1}_{\{|x-t| \le r\}} \, dx \right) dt \\
        &= \int_{0}^{1} |g'(t)|^{2} \cdot \mu\big( [t-r, t+r] \cap [0,1] \big) \, dt,
    \end{align*}
    where $\mu(\cdot)$ denotes the Lebesgue measure. For any $t \in [0,1]$, the length of the intersection is bounded by the length of the interval $[t-r, t+r]$, i.e., $\mu([t-r, t+r] \cap [0,1]) \le 2r$.
    
    Substituting this back, we obtain:
    \begin{equation*}
        \mathbb{E}_X\left[\sup_{x^{\prime} \in A(X)} |g(x^{\prime}) - g(X)|^{2}\right] \lesssim r \int_{0}^{1} |g'(t)|^{2} (2r) \, dt = 2r^{2} \|g'\|_{L^{2}}^{2}.
    \end{equation*}
    Recalling the definition of the $H^{1}$ seminorm $\|g\|_{H^{1}} = \|g'\|_{L^{2}}$, and noting that $\|g\|_{H^{1}} \le \|g\|_{H^{s}}$ for any $s > 1$, we conclude:
    \begin{equation*}
        \mathbb{E}_X\left[\sup_{x^{\prime} \in A(X)} |g(x^{\prime}) - g(X)|^{2}\right] \lesssim 2 r^{2} \|g\|_{H^{1}}^{2} \le C r^{2} \|g\|_{H^{s}}^{2}.
    \end{equation*}
\end{proof}
Combining Lemma A.4, Lemma A.5, and Lemma A.6, we obtain the final result:
\begin{theorem}[Modulus of Continuity in $H^{s}$]
There is a constant $C=C_s$ such that for any function $g \in H^{s}([0,1])$ where $s> \frac{1}{2}$, the local variation satisfies:
\begin{align*}
    \mathbb{E}_X\left[\sup_{x^{\prime} \in A(X)} |g(x^{\prime}) - g(X)|^{2}\right] \le r^{2(1 \wedge s)} \|g\|_{H^{s}}^{2}.
\end{align*}
\end{theorem}
Combining Theorem A.6 and Corollary A.2, we obtain the desired bound:
\begin{align*}
    R_A(\hat{f}, f) \lesssim \mathbb{E}\left[\sup_{x^{\prime} \in A(X)} |\hat{f}(x^{\prime}) - \hat{f}(X)|^{2}\right] + \mathbb{E}|\hat{f} - f^*| \leq r^{2(1 \wedge s)} + n^{\frac{-2s}{2s+1}}.
\end{align*}
\subsubsection{Proof of $d>1$, $s>d/2$}

\begin{lemma}[Local Riesz Potential Estimate]
\label{lem:riesz_potential_local}
Let $B(x, r) \subset \mathbb{R}^d$ be an open ball of radius $r$ centered at $x$, and let $u \in W^{1,1}(B(x, r))$. Then, for the average value $u_{B} = \frac{1}{|B|} \int_{B} u(y) \, dy$ and almost every $x' \in B(x,r)$, the following estimate holds:
\begin{align*}
    |u(x') - u_{B}| \le \frac{2^d}{d \omega_d} \int_{B(x,r)} \frac{|\nabla u(y)|}{|x-y|^{d-1}} \, dy,
\end{align*}
where $\omega_d$ denotes the volume of the unit ball in $\mathbb{R}^d$, and the constant depends only on the dimension $d$.
\end{lemma}

\begin{proof}[Remark]
This is a direct adaptation of Lemma 7.16 in \cite{gilbarg2001elliptic}, taking $\Omega = S = B(x,r)$.
\end{proof}

Using the potential estimate above, we can now derive the local Hölder continuity for Sobolev functions when $q > d$.

\begin{lemma}[Local Morrey Estimate]
\label{lem:local_morrey}
Let $B = B(x,r) \subset \mathbb{R}^d$ be a ball centered at $x$. Assume $u \in W^{1,q}(B)$ with $q > d$. Then for almost every $x' \in B$, the following estimate holds:
\begin{align*}
    |u(x') - u(x)| \le C_{d,q} \, r \left( \frac{1}{|B|} \int_{B} |\nabla u(y)|^q \, dy \right)^{1/q},
\end{align*}
where $C_{d,q}$ is a constant depending only on $d$ and $q$.
\end{lemma}

\begin{proof}
By the triangle inequality, we have
\begin{align*}
    |u(x') - u(x)| \le |u(x') - u_B| + |u(x) - u_B|,
\end{align*}
where $u_B$ is the average of $u$ over $B$.
Since the ball $B$ is convex, the Riesz potential estimate from Lemma \ref{lem:riesz_potential_local} applies to any point $z \in B$ (including $x$ and $x'$):
\begin{align*}
    |u(z) - u_B| \le C_d \int_{B} \frac{|\nabla u(y)|}{|z-y|^{d-1}} \, dy.
\end{align*}
We estimate the integral term using Hölder's inequality with exponents $q$ and $q' = \frac{q}{q-1}$:
\begin{align*}
    \int_{B} \frac{|\nabla u(y)|}{|z-y|^{d-1}} \, dy \le \left( \int_B |\nabla u(y)|^q \, dy \right)^{1/q} \left( \int_B |z-y|^{-(d-1)q'} \, dy \right)^{1/q'}.
\end{align*}
The second integral involves a singularity at $z$. Since $z \in B(x,r)$, we have $B(x,r) \subset B(z, 2r)$. Extending the domain of integration to $B(z, 2r)$ yields:
\begin{align*}
    \int_B |z-y|^{-(d-1)q'} \, dy \le \int_{B(z, 2r)} |z-y|^{-(d-1)q'} \, dy = \omega_d \int_0^{2r} \rho^{d-1 - (d-1)q'} \, d\rho.
\end{align*}
The exponent of $\rho$ is
\[
d - 1 - (d-1)\frac{q}{q-1} = \frac{(d-1)(q-1) - (d-1)q}{q-1}  = \frac{-(d-1)}{q-1} = \frac{q-d}{q-1} - 1.
\]
Since $q > d$, the exponent is strictly greater than $-1$, so the integral converges. Evaluating it gives:
\begin{align*}
    \left( \int_{B(z, 2r)} |z-y|^{-(d-1)q'} \, dy \right)^{1/q'} \le C \left( (2r)^{\frac{q-d}{q-1}} \right)^{\frac{q-1}{q}} = C' r^{\frac{q-d}{q}} = C' r^{1 - d/q}.
\end{align*}
Recalling that $|B| = \omega_d r^d$, we can write $r^{-d/q} = C |B|^{-1/q}$. Thus,
\begin{align*}
    r^{1 - d/q} = r \cdot r^{-d/q} \le C r |B|^{-1/q}.
\end{align*}
Substituting this back into the Hölder estimate, we obtain for any $z \in B$:
\begin{align*}
    |u(z) - u_B| \le C_{d,q} \, r \left( \frac{1}{|B|} \int_B |\nabla u(y)|^q \, dy \right)^{1/q}.
\end{align*}
Summing the estimates for $z=x$ and $z=x'$ completes the proof.
\end{proof}

\begin{theorem}
Let $d \geq 2$. Suppose $f \in H^s([0,1]^d)$ with $s > d/2$. There exists a constant $C$ such that for any $x \in [0,1]^d$ and $r > 0$:
\begin{align*} \label{eq:target}
    \int_{[0,1]^d} \sup_{x' \in B(x,r)} |f(x') - f(x)|^2 \, dx \leq C r^2 \|f\|_{H^s}^2.
\end{align*}
\end{theorem}

\begin{proof}
Since $s > d/2$, the Sobolev embedding theorem ensures $H^s([0,1]^d) \hookrightarrow W^{1,p}([0,1]^d)$ for some $p > d$. Specifically, the condition $s - d/2 \ge 1 - d/p$ yields the estimate:
\begin{align*}
    \|\nabla f\|_{L^p([0,1]^d)} \le C \|f\|_{H^s([0,1]^d)}.
\end{align*}
For a ball $B = B(x,r)$ and any $x' \in B$, the local Morrey estimate implies:
\begin{align*}
    |f(x') - f(x)| \leq C_{d,q} \, r \left( \frac{1}{|B|} \int_{B} |\nabla f(y)|^q \, dy \right)^{1/q},
\end{align*}
where $q$ is an appropriate exponent. In terms of the Hardy-Littlewood maximal function $M$, we have the pointwise bound:
\begin{align*} 
    |f(x') - f(x)| \leq C r \left[ M(|\nabla f|^q)(x) \right]^{1/q}.
\end{align*}
Squaring the above inequality, taking the supremum over $x' \in B(x,r)$, and integrating over $[0,1]^d$ gives:
\[
    \int_{[0,1]^d} \sup_{x' \in B(x,r)} |f(x') - f(x)|^2 \, dx \leq C r^2 \int_{[0,1]^d} \left[ M(|\nabla f|^q)(x) \right]^{2/q} \, dx.
\]
Let $h = |\nabla f|^q$. Since $\nabla f \in L^p$, it follows that $h \in L^{p/q}([0,1]^d)$. We estimate the term $\| Mh \|_{L^{2/q}}^{2/q}$.
Since $d \geq 2$ and $p > d$, we have $p > 2$, which implies $p/q > 2/q$. By Hölder's inequality on the bounded domain $[0,1]^d$, we obtain the inclusion:
\begin{align*}
    \| Mh \|_{L^{2/q}([0,1]^d)} \leq C \| Mh \|_{L^{p/q}([0,1]^d)}.
\end{align*}
Furthermore, since $p/q > 1$, the maximal operator $M$ is bounded on $L^{p/q}([0,1]^d)$:
\begin{align*}
    \| Mh \|_{L^{p/q}([0,1]^d)} \leq C \| h \|_{L^{p/q}([0,1]^d)} = C \|\nabla f\|_{L^p([0,1]^d)}^q.
\end{align*}
Combining these estimates concludes the proof:
\begin{align*}
    \int_{[0,1]^d} \sup_{x' \in B(x,r)} |f(x') - f(x)|^2 \, dx &\leq C r^2 \left( \| Mh \|_{L^{2/q}} \right)^{2/q} \\
      &\leq C r^2 \left( C \| h \|_{L^{p/q}} \right)^{2/q} \\
      &= C' r^2 \|\nabla f\|_{L^p}^2 \\
      &\leq C'' r^2 \|f\|_{H^s}^2.
\end{align*}
\end{proof}
The same argument as in the case $d=1$ shows that this estimator has the error rate:
\begin{align*}
    \E[R_A(\hat{f}, f)] \leq r^2 + n^{-\frac{2s}{2s+d}}.
\end{align*}
\subsubsection{Proof of $s<d/2$}
As in \cite{zhang2024optimality}, the standard KRR estimator $\tilde{f}$ satisfies $\E \|\tilde{f}-f^*\|^2_{H^s}\lesssim 1$, $\E |\tilde{f}-f^*|_{L^2}\lesssim n^{\frac{-s}{2s+d}}$.
Now we introduce a lemma as our key tool for the argument.
\begin{lemma}
    Suppose $\phi \in C^1(\R^d)\cap L^\infty(\R^d)$ and satisfies $\operatorname{supp}(\mathcal{F}(\phi)) \subset B\left(0,\frac{1}{r}\right)$. Then $\phi \in C^\infty(\mathbb{R}^d)$ and 
    \begin{equation*}
        \int_{\mathbb{R}^d} \sup_{y \in B(x,r)}|\phi(y)|^2 dx \leq C \|\phi\|_{L^2}^2.
    \end{equation*}
\end{lemma}

\begin{proof}
    Let $\varphi(x)=\phi(rx)$. By changing variables, we have $\|\varphi\|_{L^2}^2 = r^{-d} \|\phi\|_{L^2}^2$, so it suffices to prove:
    \begin{equation*}
        \int_{\mathbb{R}^d} \sup_{y \in B(x,1)}|\varphi(y)|^2 dx \leq C \|\varphi\|_{L^2}^2.
    \end{equation*}

    By Lemma 6.5.3 of \cite{grafakos2014modern}, for any $p > 0$, there exists a constant $c = c_p$ such that 
    \begin{equation*}
        \sup_{z \in \mathbb{R}^d} \frac{|\varphi(x-z)|}{1 + |z|^{\frac{d}{p}}} \leq c \left[ \mathcal{M}(|\varphi|^p)(x) \right]^{\frac{1}{p}},
    \end{equation*}
    where $\mathcal{M}$ is the Hardy-Littlewood maximal function, and $d$ is the dimension.

    For any $y \in B(x,1)$, let $z = x - y$. Since $|z| \leq 1$, we have $1 + |z|^{\frac{d}{p}} \leq 2$. Therefore,
    \begin{equation*}
        \sup_{y \in B(x,1)} |\varphi(y)| \leq 2 \sup_{z \in \mathbb{R}^d} \frac{|\varphi(x-z)|}{1 + |z|^{\frac{d}{p}}} \leq 2c \left[ \mathcal{M}(|\varphi|^p)(x) \right]^{\frac{1}{p}}.
    \end{equation*}

    Squaring both sides and integrating over $\mathbb{R}^d$, we get:
    \begin{align*}
        \int_{\mathbb{R}^d} \sup_{y \in B(x,1)}|\varphi(y)|^2 dx &\leq C' \int_{\mathbb{R}^d} \left[ \mathcal{M}(|\varphi|^p)(x) \right]^{\frac{2}{p}} dx \\
        &= C' \|\mathcal{M}(|\varphi|^p)\|_{L^{\frac{2}{p}}}^{\frac{2}{p}}.
    \end{align*}

    By choosing $p=3/2$, we have $2/p = 4/3 > 1$. Since the Hardy-Littlewood maximal operator $\mathcal{M}$ is strongly bounded on $L^q$ for all $q > 1$, we obtain:
    \begin{align*}
        \|\mathcal{M}(|\varphi|^p)\|_{L^{\frac{2}{p}}}^{\frac{2}{p}} &\leq C'' \| |\varphi|^p \|_{L^{\frac{2}{p}}}^{\frac{2}{p}} \\
        &= C'' \int_{\mathbb{R}^d} \left(|\varphi(x)|^p\right)^{\frac{2}{p}} dx \\
        &= C'' \|\varphi\|_{L^2}^2.
    \end{align*}
    
    Combining these inequalities, we conclude the proof.
\end{proof}
Now we establish the adversarial property of the truncated function.
\begin{lemma}
    Suppose that $\|f\|_{H^s} \leq R$, and let $g = \mathcal{F}^{-1}(\mathcal{F}(f)\chi_r)$, where $\chi_r(\xi) = \chi(r\xi)$ and $\chi$ is a smooth function supported on $B(0,1)$, taking values in $[0,1]$, and identically $1$ on $B(0,1/2)$. Then there exists a constant $C > 0$ such that
    \begin{equation*}
        \int_{\mathbb{R}^d} \sup_{x'\in B(x,r)}|g(x')-g(x)|^2 dx \leq C R^2 r^{2\min(1,s)}.
    \end{equation*}
\end{lemma}

\begin{proof}
    For any $x' \in B(x,r)$, by the Mean Value Theorem, we have $|g(x') - g(x)| \leq |x' - x| \sup_{y \in B(x,r)} |\nabla g(y)| \leq r \sup_{y \in B(x,r)} |\nabla g(y)|$. Squaring both sides and integrating over $\mathbb{R}^d$, we get
    \begin{equation*}
        \int_{\mathbb{R}^d} \sup_{x'\in B(x,r)}|g(x')-g(x)|^2 dx \leq \int_{\mathbb{R}^d} r^2 \sup_{y\in B(x,r)}|\nabla g(y)|^2 dx \leq r^2 \sum_{i=1}^d \int_{\mathbb{R}^d} \sup_{y\in B(x,r)}|\nabla_i g(y)|^2 dx.
    \end{equation*}

    Notice that $\mathcal{F}(\nabla_i g)(\xi) = i\xi_i\mathcal{F}(f)(\xi)\chi_r(\xi)$. Since $\operatorname{supp}(\chi_r) \subset B\left(0,\frac{1}{r}\right)$, the function $\mathcal{F}(\nabla_i g)$ is also supported on $B\left(0,\frac{1}{r}\right)$. Thus, by Lemma A.11, there exists a constant $C_1$ such that
    \begin{equation*}
        r^2 \sum_{i=1}^d \int_{\mathbb{R}^d} \sup_{y\in B(x,r)}|\nabla_i g(y)|^2 dx \leq C_1 r^2 \sum_{i=1}^d \|\nabla_i g\|_{L^2}^2 = C_1 r^2 \|\nabla g\|_{L^2}^2.
    \end{equation*}

    By Plancherel's theorem, we can rewrite the $L^2$ norm of the gradient in the frequency domain (absorbing $(2\pi)^{-d}$ into a new constant $C_2$):
    \begin{equation*}
        C_1 r^2 \|\nabla g\|_{L^2}^2 = C_2 r^2 \int_{\R^d} \sum_{i=1}^d |\xi_i|^2 |\mathcal{F}(g)(\xi)|^2 d\xi \leq C_2 r^2 \int_{B\left(0,\frac{1}{r}\right)} |\xi|^2 |\mathcal{F}(f)(\xi)|^2 d\xi.
    \end{equation*}

    Now, we bound the integral based on the value of $s$, using the condition $\|f\|_{H^s}^2 = \int_{\mathbb{R}^d} (1+|\xi|^2)^s |\mathcal{F}(f)(\xi)|^2 d\xi \leq R^2$:

    \textbf{Case 1: $s \geq 1$.}
    
    Since $|\xi|^2 \leq (1+|\xi|^2)^1 \leq (1+|\xi|^2)^s$, we have:
    \begin{equation*}
        \int_{B\left(0,\frac{1}{r}\right)} |\xi|^2 |\mathcal{F}(f)(\xi)|^2 d\xi \leq \int_{B\left(0,\frac{1}{r}\right)} (1+|\xi|^2)^s |\mathcal{F}(f)(\xi)|^2 d\xi \leq R^2.
    \end{equation*}
    Multiplying by $C_2 r^2$, the bound is $C_2 R^2 r^2$.

    \textbf{Case 2: $0 \leq s < 1$.}
    
    We rewrite the integral to extract the Sobolev weight:
    \begin{equation*}
        \int_{B\left(0,\frac{1}{r}\right)} |\xi|^2 |\mathcal{F}(f)(\xi)|^2 d\xi = \int_{B\left(0,\frac{1}{r}\right)} \frac{|\xi|^2}{(1+|\xi|^2)^s} (1+|\xi|^2)^s |\mathcal{F}(f)(\xi)|^2 d\xi.
    \end{equation*}
    Since the function $t \mapsto \frac{t}{(1+t)^s}$ is strictly increasing for $t>0$, its supremum on the domain $|\xi|^2 < \frac{1}{r^2}$ is achieved exactly at the boundary:
    \begin{equation*}
        \sup_{|\xi|<\frac{1}{r}} \frac{|\xi|^2}{(1+|\xi|^2)^s} \leq \sup_{|\xi|<\frac{1}{r}} \frac{|\xi|^2}{(|\xi|^2)^s} = \sup_{|\xi|<\frac{1}{r}} |\xi|^{2-2s} = r^{2s-2}.
    \end{equation*}
    Therefore:
    \begin{align*}
        \int_{B\left(0,\frac{1}{r}\right)} |\xi|^2 |\mathcal{F}(f)(\xi)|^2 d\xi &\leq r^{2s-2} \int_{B\left(0,\frac{1}{r}\right)} (1+|\xi|^2)^s |\mathcal{F}(f)(\xi)|^2 d\xi \\
        &\leq R^2 r^{2s-2}.
    \end{align*}
    Multiplying by $C_2 r^2$ from outside the integral, the total bound becomes $C_2 r^2 \cdot R^2 r^{2s-2} = C_2 R^2 r^{2s}$.

    Combining both cases and letting $C = C_2$, we conclude that the integral is bounded by $C R^2 r^{2\min(1,s)}$.
\end{proof}
Now we can finally prove the result for $s < d/2$. 

First, we bound $\|\hat{f} - \bar{f}\|_{L^2}$:
\begin{align*}
    \|\hat{f} - \bar{f}\|_{L^2}^2 &= \int |\mathcal{F}(\bar{f})(1-\chi_r)(\xi)|^2 d\xi \\
    &\leq \int_{B(0,1/(2r))^c} \frac{1}{(1+|\xi|^2)^s}(1+|\xi|^2)^s|\mathcal{F}(\bar{f})(\xi)|^2 d\xi \\
    &\leq r^{2s}\|T\|^2_{H^s} \leq R^2 r^{2s}.
\end{align*}

Since $\|\bar{f}\|_{H^s} \leq R$, it follows that $\mathbb{E}\|\bar{f} - f^*\|^2_{L^2} \leq \mathbb{E}\|\bar{f} - f^*\|^2_{L^2} + \mathbb{E}\|T - \bar{f}\|^2_{L^2} \leq r^{2\min(1,s)} + n^{-\frac{2s}{2s+d}}$. Therefore, we have:
\begin{align*}
    R_A(\hat{f},f^*) &\lesssim \int_{\mathbb{R}^d} \sup_{x'\in B(x,r)}|\hat{f}(x')-\hat{f}(x)|^2 dx + \mathbb{E}\|\hat{f}-\bar{f}\|^2_{L^2} +  \mathbb{E}\|f^*-\bar{f}\|^2_{L^2} \\
    &\lesssim r^{2\min(1,s)} + n^{-\frac{2s}{2s+d}}.
\end{align*}

Here, Lemma A.12 can be applied to $\hat{f}$ because $\bar{f}\in L^2$ and $\mathcal{F}^{-1}(\chi_r)$ is in the Schwartz class. Since $\chi_r$ is in the Schwartz class, the convolution of a function in $L^2$ and a function in the Schwartz class gives a function in $L^\infty(\R^d) \cap C^\infty(\R^d)$. Furthermore, the Sobolev norm of $\hat{f}$ is bounded as follows:
\begin{align*}
    \|\hat{f}\|_{H^s}^2 &= \int_{\mathbb{R}^d}(1+|\xi|^2)^s |\mathcal{F}(\bar{f})(\xi)|^2\chi^2_r(\xi) d\xi \leq \|\bar{f}\|^2_{H^s}.
\end{align*}
\subsubsection{Proof of the Adaptive Upper Bound}
We summarize the bound from \cite{zhang2024optimality} as follows: 
If $v \leq n^a$ for some $a < \frac{2s_{\max}}{d}$, then for a given $0 < s < s_{max}$, there exist constants $M$ and $C = C_{s_{max}}$ such that for any $n>M$, with probability at least $1 - \delta$, we have:
\begin{align*}
    \|\hat{f}_v - f^*\|^2_{H^s} &\leq C\left(\log\frac{6}{\delta}\right)^2 \left(1 + \frac{v^{\frac{2s+d}{2s_{\max}}}}{n}\right), \\
    \|\hat{f}_v - f^*\|^2_{L^2} &\leq C\left(\log\frac{6}{\delta}\right)^2 \left(v^{-\frac{s}{s_{\max}}} + \frac{v^{\frac{d}{2s_{\max}}}}{n}\right).
\end{align*}

This result implies that, for any given $s \in [s_{\min}, s_{\max}]$, with probability at least $1 - T\log(n)\delta$ (where $T = T_{s_{\max}}$ is a constant), the following inequalities hold for all $v_k$ (let $A$ denote this event):
\begin{align*}
    \|\hat{f}_{v_k} - f^*\|^2_{H^s} &\leq C\left(\log\frac{6}{\delta}\right)^2 \left(1 + \frac{v_k^{\frac{2s+d}{2s_{\max}}}}{n}\right), \\
    \|\hat{f}_{v_k} - f^*\|^2_{L^2} &\leq C\left(\log\frac{6}{\delta}\right)^2 \left(v_k^{-\frac{s}{s_{\max}}} + \frac{v_k^{\frac{d}{2s_{\max}}}}{n}\right).
\end{align*}
Setting $\delta = n^{-3s_{\max}}$, we first bound the expectation over the complement event $A^c$:
\begin{align*}
    \mathbb{E}_{A^c} \left[ \int_{[0,1]^d} \sup_{x' \in B(x,r)} |\hat{f}(x') - f(x)|^2 \, \mathrm{d}x \right] \leq (\log n) \mathbb{P}(A^c) \lesssim n^{-3s_{\max}} \log^2 n.
\end{align*}

Let $v_{\mathrm{op}} = \mathop{\arg\min}_{v_k:\,1 \leq k \leq K} |v_k - n^{\frac{2s_{\max}}{2s+d}}|$. Now we prove that under event $A$, $v_{\mathrm{op}}$ satisfies the selection condition. An easy calculation of the trade-off between the two terms shows that $\|f_{v_k}-f^*\|_{L^2} \leq C\log\left(\frac{6}{\delta}\right) (\approx \log n) \frac{v_k^{\frac{d}{4s_{\max}}}}{\sqrt{n}}$ for all $v_k \geq v_{\mathrm{op}}$. Thus, the triangle inequality shows that $\|f_{v_k}-f_{v_{\mathrm{op}}}\|_{L^2} \leq \|f_{v_k}-f^*\| + \|f^*-f_{v_{\mathrm{op}}}\| \lesssim \log(n) \left(\frac{v_k^{\frac{d}{4s_{\max}}}}{\sqrt{n}} + \frac{v_{\mathrm{op}}^{\frac{d}{4s_{\max}}}}{\sqrt{n}}\right) \leq \log(n)\left(\frac{v_k^{\frac{d}{4s_{\max}}}}{\sqrt{n}}\right)$ for all $v_k > v_{\mathrm{op}}$, which satisfies the condition.

Thus, conditional on event $A$, the selected index $v^*$ must satisfy $v^* \leq v_{\mathrm{op}}$. This yields:
\begin{align*}
    \|\hat{f}_{v_{\mathrm{op}}} - \hat{f}_{v^*}\|_{L^2} &\leq \log^2(n)\frac{v_{\mathrm{op}}^{\frac{d}{4s_{\max}}}}{\sqrt{n}}, \\
    \|\hat{f}_{v^*} - f^*\|_{H^s} &\leq \log^2(n).
\end{align*}

According to Lemma A.12, under event $A$, since $|\hat{f}_{v^*}| = O_n(\log^3(n))$, we have $\int \sup_{x' \in B(x,r)}|\hat{f}(x') - \hat{f}(x)|^2 \, \mathrm{d}x \lesssim \log^6(n)r^{\min(1,s)}$ and $\|\hat{f} - \hat{f}_{v^*}\|_{L^2} \lesssim \log^2(n)n^{\frac{-s}{2s+d}}$.

The final minimax rate can be split over $A^c$ and $A$. When event $A$ occurs, $\int \sup_{x' \in B(x,r)}|\hat{f}(x') - f(x)|^2 \leq \log^4(n)\left(r^{2\min(1,s)} + n^{\frac{-2s}{2s+d}}\right)$. When the event does not occur, with probability $T\log(n)n^{-3s_{\max}}$, since the loss is bounded by $\log(n)$ pointwise, the final risk can be bounded by $T\log^2(n)n^{-3s_{\max}}$. Combining these intermediate results, we obtain the final minimax rate.
\subsection{Lower Bound}
We adapt the notation in \cite{peng2024adversarial} and define:
\begin{align*}
    G_{A}(f) \triangleq \frac{1}{2} \left\{ \int_{[0,1]^d} \left[ \sup_{x' \in B(x,r)} f(x') - \inf_{x' \in B(x,r)} f(x') \right]^{2} \, dx \right\}^{\frac{1}{2}}.
\end{align*}
To derive the lower bound, we construct specific base functions depending on the smoothness level $s$.

\paragraph{Case 1: $0 < s < 1$.}
We construct the base regression function $f_0$ using a periodic perturbation. First, define an auxiliary function $\phi_0$ on $[0,1]$ as:
\begin{align*}
    \phi_0(x) \triangleq \begin{cases}
        0 & 0 \le x < \frac{1}{4}, \\
        (x - \frac{1}{4})^s & \frac{1}{4} \le x < \frac{1}{2}, \\
        (1/4)^s & \frac{1}{2} \le x < \frac{3}{4}, \\
        (1 - x)^s & \frac{3}{4} \le x \le 1
    \end{cases}.
\end{align*}
Define $a_k \triangleq 8kr$ for $k = 0, \dots, K$, where $K \triangleq \lfloor 1/(8r) \rfloor > 1$. The base regression function $f_0 : [0, 1]^d \to \mathbb{R}$ is defined by:
\begin{align*}
    f_0(x) = \begin{cases}
        B_s (8r)^s \phi_0 \left( \frac{x_1 - a_{k-1}}{8r} \right) & a_{k-1} \le x_1 < a_k, \quad \text{for } k = 1, \dots, K, \\
        0 & a_K \le x_1 \le 1
    \end{cases}. \tag{4.5}
\end{align*}
\paragraph{Case 2: $s > 1$.}
In this regime, we utilize a smooth bump function. Define $\psi_0$ as:
\begin{align*}
    \psi_0(x) \triangleq \exp\left(-\frac{1}{1-(x-1)^2}\right) \mathbb{I}_{\{0 \le x \le 2\}}.
\end{align*}
We define the base function $g_0$ (distinguished from $f_0$ for clarity) as:
\begin{align*}
    g_0(x) = \begin{cases}
        0 & 0 \le x < 2r, \\
        B_s \psi_0\left(\frac{x - 2r}{1 - 4r}\right) & 2r \le x < 1 - 2r, \\
        0 & 1 - 2r \le x \le 1
    \end{cases}.
\end{align*}
\begin{theorem}
We can choose $B_s$ sufficiently small such that $\|f_0\|_{H^s}\leq 1$ and $\|g_0\|_{H^s}\leq 1$ for every $r>0$.
\end{theorem}
\begin{proof}
    We establish the uniform boundedness of the Sobolev norms for both constructions. By definition, the fractional Sobolev norm is given by $\|f\|_{H^s}^2 = \|f\|_{L^2}^2 + |f|_{H^s}^2$, where the Slobodeckij seminorm is $|f|_{H^s}^2 = \iint \frac{|f(x)-f(y)|^2}{|x-y|^{d+2s}} \mathrm{d}x \mathrm{d}y$. Since our functions depend only on the first coordinate, we can restrict our analysis to the 1D norm on $\mathbb{R}$, absorbing the transverse integration into a universal constant.

    \paragraph{Case 1: $0 < s < 1$.}
    First, we verify that the base function $\phi_0 \in H^s([0,1])$. Note that $\phi_0$ is continuous, bounded, and uniformly Lipschitz except at the points $x = 1/4$ and $x = 1$, where it exhibits a singularity of the form $x^s$. The local integrability of the seminorm around $x = 1/4$ is guaranteed since for a small $\delta > 0$,
    \begin{align*}
        \int_0^\delta \int_{-\delta}^0 \frac{x^{2s}}{|x-y|^{1+2s}} \mathrm{d}y \mathrm{d}x 
        = \int_0^\delta x^{2s} \left[ \frac{1}{2s(x-y)^{2s}} \right]_{y=-\delta}^{y=0} \mathrm{d}x 
        \le \int_0^\delta \frac{x^{2s}}{2s x^{2s}} \mathrm{d}x 
        = \frac{\delta}{2s} < \infty.
    \end{align*}
    Thus, $\|\phi_0\|_{H^s} \triangleq C_\phi < \infty$. 

    Next, we bound $\|f_0\|_{L^2}^2$. The function $f_0$ consists of $K$ disjoint copies of $\phi_0$ supported on intervals $I_k = [a_{k-1}, a_k]$ of length $8r$. We have
    \begin{align*}
        \|f_0\|_{L^2}^2 
        &= \sum_{k=1}^K \int_{I_k} \left( B_s (8r)^s \phi_0\left( \frac{x - a_{k-1}}{8r} \right) \right)^2 \mathrm{d}x \\
        &= K \cdot B_s^2 (8r)^{2s} \cdot (8r) \|\phi_0\|_{L^2}^2.
    \end{align*}
    Since $K = \lfloor 1/(8r) \rfloor \le 1/(8r)$, we have $K(8r) \le 1$. Because $8r \le 1$ and $s > 0$, $(8r)^{2s} \le 1$. Thus, $\|f_0\|_{L^2}^2 \le B_s^2 \|\phi_0\|_{L^2}^2$.

    For the seminorm $|f_0|_{H^s}^2$, we decompose the integration domain into diagonal and off-diagonal parts:
    \begin{equation*}
        |f_0|_{H^s}^2 = \sum_{k=1}^K \int_{I_k} \int_{I_k} \frac{|f_0(x)-f_0(y)|^2}{|x-y|^{1+2s}} \mathrm{d}x \mathrm{d}y + \sum_{k \neq j} \int_{I_k} \int_{I_j} \frac{|f_0(x)-f_0(y)|^2}{|x-y|^{1+2s}} \mathrm{d}x \mathrm{d}y.
    \end{equation*}
    For the diagonal terms, changing variables $u = \frac{x-a_{k-1}}{8r}$ and $v = \frac{y-a_{k-1}}{8r}$ yields:
    \begin{align*}
        \int_{I_k} \int_{I_k} \frac{|f_0(x)-f_0(y)|^2}{|x-y|^{1+2s}} \mathrm{d}x \mathrm{d}y 
        &= \int_0^1 \int_0^1 \frac{B_s^2 (8r)^{2s} |\phi_0(u)-\phi_0(v)|^2}{(8r)^{1+2s} |u-v|^{1+2s}} (8r)^2 \mathrm{d}u \mathrm{d}v \\
        &= B_s^2 (8r) |\phi_0|_{H^s}^2.
    \end{align*}
    Summing over $k=1, \dots, K$ gives strictly $K(8r) B_s^2 |\phi_0|_{H^s}^2 \le B_s^2 |\phi_0|_{H^s}^2$.

    For the off-diagonal terms, let $x \in I_k$ and $y \in I_j$ with $k \neq j$. The distance is bounded below by $|x-y| \ge 8r(|k-j|-1)$. Using $|f_0(x) - f_0(y)|^2 \le 2(f_0(x)^2 + f_0(y)^2)$ and $\|f_0\|_{L^\infty} \le B_s (8r)^s \|\phi_0\|_{L^\infty}$, we obtain:
    \begin{align*}
        \int_{I_k} \int_{I_j} \frac{|f_0(x)-f_0(y)|^2}{|x-y|^{1+2s}} \mathrm{d}x \mathrm{d}y 
        &\le \frac{2 \cdot (8r)^2 \cdot 2 B_s^2 (8r)^{2s} \|\phi_0\|_{L^\infty}^2}{(8r(|k-j|-1))^{1+2s}} \\
        &= 4 B_s^2 \|\phi_0\|_{L^\infty}^2 (8r) \frac{1}{(|k-j|-1)^{1+2s}}.
    \end{align*}
    Summing over $j \neq k$ introduces the series $\sum_{m=1}^\infty m^{-(1+2s)}$, which converges to a constant $C_s < \infty$ since $1+2s > 1$. Summing over $k$ yields $K$ identical bounds. Since $K(8r) \le 1$, the total off-diagonal sum is bounded by $4 C_s B_s^2 \|\phi_0\|_{L^\infty}^2$. Combining these, $\|f_0\|_{H^s}^2 \le C_1 B_s^2$ for some universal constant $C_1 > 0$ independent of $r$. Choosing $B_s \le 1/\sqrt{C_1}$ guarantees $\|f_0\|_{H^s} \le 1$.

    \paragraph{Case 2: $s > 1$.}
    The function $g_0$ is defined using a smooth bump function $\psi_0 \in C_c^\infty([0,2])$, specifically $g_0(x) = B_s \psi_0\left(\frac{x - 2r}{1 - 4r}\right)$. 
    Since $K = \lfloor 1/(8r) \rfloor > 1$, it is guaranteed that $8r < 1/2$, meaning $r \in (0, 1/16)$. Thus, the scaling factor $\lambda = 1 - 4r$ satisfies $\lambda \in (3/4, 1)$.
    
    Because $\psi_0$ is smooth and compactly supported, $\psi_0 \in H^s(\mathbb{R})$. The $H^s$ norm of the affinely scaled function $h(x) = \psi_0(x/\lambda)$ depends continuously on $\lambda$. Since $\lambda$ is strictly bounded away from zero and infinity (contained in the compact interval $[3/4, 1]$), there exists a uniform constant $C_2$ depending only on $s$ and $\psi_0$ such that:
    \begin{equation*}
        \left\| \psi_0\left(\frac{\cdot - 2r}{1 - 4r}\right) \right\|_{H^s(\mathbb{R})} \le C_2 \|\psi_0\|_{H^s(\mathbb{R})}.
    \end{equation*}
    Consequently, $\|g_0\|_{H^s} \le B_s C_2 \|\psi_0\|_{H^s}$. Choosing $B_s \le 1 / (C_2 \|\psi_0\|_{H^s})$ uniformly guarantees $\|g_0\|_{H^s} \le 1$ for all valid $r > 0$.
\end{proof}
The following lemma in \cite{peng2024adversarial} establishes a lower bound on the adversarial functional $G_A(\cdot)$ for any function approximating these base functions.

\begin{lemma} \label{lem:lower_bound_G}
    Consider an adversarial attack $A \in \mathcal{T}(r)$ with $0 \le r < 1/8$ and principal direction $v=(1, 0, \dots, 0)^\top$. Let $h_0$ denote either $f_0$ (if $0 < s \le 1$) or $g_0$ (if $s > 1$). 
    
    There exist constants $C_1, C_2 > 0$ such that for any estimator $\hat{f}$ satisfying
    \begin{align*}
        \|\hat{f} - h_0\|_2 \le C_1 r^{1 \wedge s},
    \end{align*}
    the functional $G_{A}(\hat{f})$ is lower bounded by:
    \begin{align*}
        G_{A}(\hat{f}) \ge C_2 r^{1 \wedge s}.
    \end{align*}
\end{lemma}
\begin{proof}[Proof of Theorem 3.1]
    Since $R_A(\hat{f},f)\geq G_A(\hat{f})^2$ and $R_A(\hat{f},f)\geq |f-\hat{f}|_{L^2}^2$, we establish the lower bound $r^{2(1\wedge s)}$, and the standard term comes from the fact that $R_A(\hat{f},f)\geq \E\|f-\hat{f}\|_{L^2}^2$.
\end{proof}
\subsection{Proof of the Impossibility of Adaptivity to $r$}
We establish this via proof by contradiction. Suppose there exists an adaptive estimator $\hat{f}_n$ (whose construction is independent of $r$) that achieves the minimax optimal rate $R_{A}(\hat{f}_n, f^*) \le C(r^{2s} + n^{-\frac{2s}{2s+d}})$ for all $r \in (0, r_0]$ and $f^* \in \mathcal{H}^{d/2}(L)$. For $s \le d/2$, the Sobolev space admits essentially unbounded functions. Fix an interior point $x_0 \in (0,1)^d$ and construct the true target function $f^*(x) = \phi(x) \log(-\log|x - x_0|)$, where $\phi$ is a smooth cutoff localized at $x_0$. Since the adversarial risk is lower bounded by the standard $L^2$ risk, taking $r \to 0$ implies $\mathbb{E}\|\hat{f}_n - f^*\|_{L^2}^2 \le C n^{-\frac{2s}{2s+d}}$.

Let $U = B(x_0, r/2)$ and define the local essential supremum $S_n = \sup_{x' \in U} |\hat{f}_n(x')|$. Since $f^*$ is essentially unbounded on $U$, for any arbitrarily large constant $M > 0$, the truncation error lower bound $C_M = \inf_{\|g\|_{L^\infty(U)} \le M} \|g - f^*\|_{L^2(U)}^2$ is strictly positive. By Markov's inequality, $\mathbb{P}(S_n \le M) \le \mathbb{P}(\|\hat{f}_n - f^*\|_{L^2(U)}^2 \ge C_M) \le (C n^{-\frac{2s}{2s+d}}) / C_M \to 0$ as $n \to \infty$. Thus, $S_n \xrightarrow{\mathbb{P}} \infty$.

To evaluate the adversarial risk, note that for any $X \in U$, the adversarial neighborhood $B(X, r)$ strictly contains $U$, giving $\sup_{x' \in B(X,r)} |\hat{f}_n(x')| \ge S_n$. By the triangle inequality, $R_{A}(\hat{f}_n, f^*) \ge \mathbb{E}_{D_n} [\int_U (S_n - |f^*(X)|)_+^2 dX]$. To rigorously handle the expectation without assuming the integrability of $S_n$, we introduce a spatial truncation $U_M = \{X \in U : |f^*(X)| \le M/2\}$ with Lebesgue measure $V_M$. Restricting the integral to $U_M$ and the probability space to the event $\{S_n > M\}$, we obtain:
\begin{align*}
    R_{A}(\hat{f}_n, f^*) \ge \mathbb{E}_{D_n} \left[ \mathbf{1}_{\{S_n > M\}} \int_{U_M} \left( M - \frac{M}{2} \right)^2 dX \right] = \frac{M^2}{4} V_M \mathbb{P}(S_n > M).
\end{align*}
Fixing $r$ and taking $n \to \infty$ yields $\liminf_{n \to \infty} R_{A}(\hat{f}_n, f^*) \ge \frac{M^2}{4} V_M$. Since $f^* \in L^2$, we have $V_M \to \text{Vol}(U) > 0$ as $M \to \infty$. As $M$ can be chosen arbitrarily large, we conclude $\liminf_{n \to \infty} R_{A}(\hat{f}_n, f^*) = \infty$. This contradicts our initial assumption that $\lim_{n \to \infty} R_{A}(\hat{f}_n, f^*) \le C r^{2s} < \infty$, demonstrating that adaptivity to $r$ is unachievable.

\section{Proof of section 4}

\subsection{Proof of lemma 4.1}

According to Lemma C.8 and the proof of Lemma C.10 in \cite{chen2024impacts}, for any non-empty bounded domain $\mathcal{X} \subset \mathbb{R}^d$ with a $C^\infty$ smooth boundary, the following isomorphisms are valid.

First, let $\phi(x) = \frac{(x,1)}{\|(x,1)\|}$ be the mapping from $\mathcal{X}$ to the upper hemisphere, defining the subdomain $S = \phi(\mathcal{X}) \subset \mathbb{S}_+^d$. Let $\phi_1$ be the stereographic projection and $\mathcal{X}_1 = \phi_1(S)$. By Lemma C.8, we have the first isomorphism regarding the homogeneous NTK on the spherical subdomain:
\[
    \mathcal{H}_0^{NT}(S) \cong H^{\frac{d+1}{2}}(\mathcal{X}_1) \circ \phi_1.
\]

Second, the proof of Lemma C.10 explicitly constructs an isometric isomorphism between the RKHS on the original domain $\mathcal{X}$ and the RKHS on the spherical subdomain $S$. By defining the multiplier $\rho(x) = \|(x,1)\|$, the mapping $I: f \mapsto \rho \cdot (f \circ \phi)$ guarantees the second isomorphism:
\[
    \mathcal{H}^{NT}(\mathcal{X}) \cong \mathcal{H}_0^{NT}(S).
\]

By composing these mappings, we obtain the explicit structural equivalence:
\[
    \mathcal{H}^{NT}(\mathcal{X}) \cong \left( H^{\frac{d+1}{2}}(\mathcal{X}_1) \circ \phi_1 \right) \circ \phi \cdot \rho.
\]

To rigorously establish that this equates to the standard fractional Sobolev space $H^{\frac{d+1}{2}}(\mathcal{X})$, we must verify the boundedness of the induced operators and their inverses. Since $\mathcal{X}$ is a bounded domain with a $C^\infty$ smooth boundary, its closure $\overline{\mathcal{X}}$ is compact. The composite mapping $\Phi = \phi_1 \circ \phi : \mathcal{X} \to \mathcal{X}_1$ is a smooth diffeomorphism between bounded domains. Consequently, all its derivatives of any order are uniformly bounded on the compact closure $\overline{\mathcal{X}}$, and its Jacobian determinant is strictly bounded away from zero and infinity (i.e., $c \le |J\Phi| \le C$).

Furthermore, the multiplier $\rho(x)$ is strictly positive, bounded, and $C^\infty$ smooth with all partial derivatives of each order bounded on $\overline{\mathcal{X}}$. In the theory of Sobolev spaces, the pullback via a diffeomorphism with uniformly bounded derivatives and a non-degenerate Jacobian, along with multiplication by a smooth, strictly positive function with all partial derivatives of each order bounded, constitutes a bounded invertible operator. Therefore, these operations preserve the Sobolev norm equivalence, which establishes:
\[
    \mathcal{H}^{NT}(\mathcal{X}) \cong H^{\frac{d+1}{2}}(\mathcal{X}).
\]

Since the $K^{NTK}$ in \cite{chen2024impacts} and the NTK kernel in this paper are different up to adding 1, and noticing that 1 lies in the Sobolev space, for the NTK kernel in our setting, we also have that the RKHS of NTK in a bounded domain with smooth boundary is a Sobolev class. Thus, the NTK kernel and the exponential kernel $k(x,y)=e^{-|x-y|}$ are equivalent in a bounded smooth boundary domain, also in its subdomain. For $[0,1]^d$, it can be contained in a larger bounded smooth boundary domain, and by Example 2.6 in \cite{kanagawa2018gaussian}, the RKHS of the exponential kernel is Sobolev $H^{(d+1)/2}([0,1]^d)$ (since it is Lipschitz), thus the RKHS of the NTK kernel is also a Sobolev class.

\subsection{Proof of theorem 4.2}

Leveraging the result for the gradient flow spectral algorithm in \cite{zhang2024optimality}, $\E\|f_{t^*}^{NN}\|_H = O(1)$. By Theorem A.10 and Theorem A.7, we have
\begin{align*}
    \E[\sup_{x'\in B(X,r)\cap [0,1]^d}|\hat{f}(x')-f(X)|^2]\lesssim r^{2\min(1,s)}.
\end{align*}
Combining this with Proposition 13 on $L^2$ loss in \cite{li2024eigenvalue} implies that $\hat{f}^{NN}$ attains the minimax optimal adversarial risk.

\section{Proof of section 5}
Since the true function lies in the Sobolev ellipsoid, its $L^{\infty}$ norm is bounded by a constant $C$.
\begin{definition}[Local Conditional Range]
    Fix $x \in [r, 1-r]^d$. Let $I_x \coloneqq B(x,r)$. Define the index set of data points in $I_x$ as $\mathcal{I}(x) \coloneqq \{ i : x_i \in I_x \}$ and the count $N(x) \coloneqq |\mathcal{I}(x)|$. Define the squared conditional range $V(x)$ as:
    \begin{align*}
        V(x) \coloneqq \begin{cases}
            \left( \max_{i \in \mathcal{I}(x)} \xi_i - \min_{i \in \mathcal{I}(x)} \xi_i - 2C \right)_+^2 & \text{if } N(x) \ge 2, \\
            0 & \text{otherwise}.
        \end{cases}
    \end{align*}
\end{definition}

\begin{lemma}[Gaussian Range Lower Bound] \label{lemma:gaussian_range}
    Let $Z_1, \dots, Z_k$ be i.i.d. $\mathcal{N}(0, \sigma^2)$. Let $R_k \coloneqq \max_{1\le j \le k} Z_j - \min_{1\le j \le k} Z_j - 2C$. There exist absolute constants $c_1, c_2 > 0$ such that for all $k \ge 2$:
    \begin{align*}
        \mathbb{E}[(R_k)_+^2] \ge c_1 \sigma^2 \log k - c_2.
    \end{align*}
\end{lemma}

\begin{proof}
    Let $W_k \coloneqq \max_{1\le j \le k} Z_j - \min_{1\le j \le k} Z_j$ denote the range of the Gaussian variables. 
    By the elementary inequality $(x-y)_+^2 \ge \frac{1}{2}x^2 - y^2$, we have:
    \begin{align*}
        (R_k)_+^2 = (W_k - 2C)_+^2 \ge \frac{1}{2}W_k^2 - 4C^2.
    \end{align*}
    Taking the expectation and applying Jensen's inequality ($\mathbb{E}[W_k^2] \ge (\mathbb{E}[W_k])^2$), we obtain:
    \begin{align*}
        \mathbb{E}[(R_k)_+^2] \ge \frac{1}{2}\mathbb{E}[W_k^2] - 4C^2 \ge \frac{1}{2}(\mathbb{E}[W_k])^2 - 4C^2.
    \end{align*}
    By the symmetry of the centered Gaussian distribution, the expectation of the range is $\mathbb{E}[W_k] = 2\mathbb{E}[\max_{1\le j \le k} Z_j]$. 
    A standard result for the maximum of i.i.d. Gaussian variables states that there exists an absolute constant $c_0 > 0$ such that for all $k \ge 2$:
    \begin{align*}
        \mathbb{E}\left[\max_{1\le j \le k} Z_j\right] \ge c_0 \sigma \sqrt{\log k}.
    \end{align*}
    Substituting this into our bound yields:
    \begin{align*}
        \mathbb{E}[(R_k)_+^2] \ge \frac{1}{2} \left( 2 c_0 \sigma \sqrt{\log k} \right)^2 - 4C^2 = 2 c_0^2 \sigma^2 \log k - 4C^2.
    \end{align*}
    Setting $c_1 = 2c_0^2$ (which is strictly an absolute constant) and defining $c_2 = 4C^2$, we conclude:
    \begin{align*}
        \mathbb{E}[(R_k)_+^2] \ge c_1 \sigma^2 \log k - c_2.
    \end{align*}
    This completes the proof.
\end{proof}

\begin{lemma}[Binomial Concentration] \label{lemma:chernoff}
    Let $N(x) \sim \mathrm{Bin}(n, s)$. Define the event $E_x \coloneqq \{ N(x) \ge ns/2 \}$. For sufficiently large $n$ such that $e^{-ns/8} \le 1/2$, we have $\mathbb{P}(E_x) \ge 1/2$.
\end{lemma}

\begin{proof}
    Let $X = N(x)$. The expectation is $\mu = \mathbb{E}[X] = ns$. 
    By the standard multiplicative Chernoff bound for the lower tail, for any $\delta \in (0, 1)$, we have:
    \begin{align*}
        \mathbb{P}(X \le (1-\delta)\mu) \le \exp\left(-\frac{\delta^2 \mu}{2}\right).
    \end{align*}
    Setting $\delta = 1/2$, the inequality simplifies to:
    \begin{align*}
        \mathbb{P}\left(X \le \frac{ns}{2}\right) \le \exp\left(-\frac{(1/2)^2 ns}{2}\right) = \exp\left(-\frac{ns}{8}\right).
    \end{align*}
    We are given the condition $e^{-ns/8} \le 1/2$. Consequently, $\mathbb{P}(X \le ns/2) \le 1/2$.
    
    For the event $E_x = \{ X \ge ns/2 \}$, we pass to its complement:
    \begin{align*}
        \mathbb{P}(E_x) = 1 - \mathbb{P}\left(X < \frac{ns}{2}\right) \ge 1 - \mathbb{P}\left(X \le \frac{ns}{2}\right) \ge 1 - \frac{1}{2} = \frac{1}{2}.
    \end{align*}
    This completes the proof.
\end{proof}

\begin{proof}[Proof of Theorem 5.1]
    For a given $x\in [r,1-r]^d$, noticing $P_x=\int_{y\in B(x,r)}f_X(y)dy\gtrsim r^d$.
    \begin{align*}
        \sup_{t \in I_x} \hat{f}(t) \ge \max_{i \in \mathcal{I}(x)} \xi_i-C \quad \text{and} \quad \inf_{t \in I_x} \hat{f}(t) \le \min_{i \in \mathcal{I}(x)} \xi_i+C.
    \end{align*}
    This implies the pointwise lower bound on the local oscillation:
    \begin{align*}
        \left( \sup_{t \in I_x} \hat{f}(t) - \inf_{t \in I_x} \hat{f}(t) \right)^2 \ge V(x).
    \end{align*}
    Applying Fubini's theorem:
    \begin{align}\label{eq:fubini_bound}
        \mathbb{E} [G_{A}^2(\hat{f})] \ge \frac{1}{2} \int_{[r,1-r]^d} \mathbb{E}[V(x)] \, \mathrm{d} x.
    \end{align}
    
    We evaluate $\mathbb{E}[V(x)]$ using the Law of Total Expectation, conditioning on the design $X$. Conditioned on $X$, $N(x)$ is deterministic and $\xi_i$ are i.i.d. Gaussian. Applying Lemma \ref{lemma:gaussian_range}:
    \begin{align*}
        \mathbb{E} [V(x) \mid X] \ge (c_1 \sigma^2 \log(N(x))-c_2) \cdot \mathbb{I}_{\{N(x) \ge 2\}}.
    \end{align*}
    Taking the expectation with respect to $X$ and restricting to the event $E_x$ (where $N(x) \ge 
    \frac{nP_x}{2}\ge 2$ for large $n$):
    \begin{align*}
        \mathbb{E} [V(x)] &= \mathbb{E}_X \left[ \mathbb{E}_\xi [V(x) \mid X] \right] \nonumber \\
        &\ge \mathbb{E}_X \left[ (c_1 \sigma^2-c_2) \log(N(x)) \cdot \mathbb{I}_{E_x} \right] \nonumber \\
        &\ge (c_1 \sigma^2 \log(nr^d)-c_2) \cdot \mathbb{P}(E_x).
    \end{align*}
    Using Lemma \ref{lemma:chernoff}, $\mathbb{P}(E_x) \ge 1/2$. Therefore, $\mathbb{E} [V(x)] \ge \frac{c_1}{2} \sigma^2 \log(nr^d)$. 
    Substituting this into the inequality:
    \begin{align*}
        \mathbb{E} [G_{A,2}^2(\hat{f})] &\ge \frac{1}{2} \int_{[r,1-r]^d} \frac{c_1}{2} \sigma^2 \log(nr^d)-\frac{c_2}{2} \, \mathrm{d} x \nonumber \\
        &= (\frac{c_1}{4} \sigma^2 \log(nr^d)-\frac{c_2}{4}) (1 - 2r)^d.
    \end{align*}
    Since $r \to 0$, $1-2r$ is bounded below by a constant (e.g., $1/2$). Thus, there exists $C' > 0$ such that $\mathbb{E} [G_{A,2}^2(\hat{f})] \ge C' \sigma^2 \log(nr^d)$. Taking the square root yields the result.
\end{proof}

\begin{proof}[Proof of Lemma 5.2]
    Since the RKHS is $H^{\frac{d+1}{2}}$, which is norm-equivalent to the RKHS of the exponential kernel $k(x,y)=e^{-\|x-y\|}$, we restrict our analysis to the kernel matrix of the exponential kernel without loss of generality.
    
    We aim to establish the upper bound for the quadratic form. We want to show that there exists a constant $A=A(d)$ such that for any $y \in S^{n-1}$ (i.e., $\|y\|_2 = 1$),
    \begin{align*}
        y^T \mathbf{K}^{-1} y \leq A q_X^{-1},
    \end{align*}
    where $q_X = \frac{1}{2} \min_{i \neq j} \|x_i - x_j\|$ is the separation distance (packing radius).
    
    Recall the variational characterization of the quadratic form in RKHS:
    \begin{align*}
        y^T \mathbf{K}^{-1} y = \inf \left\{ \|f\|_{\mathcal{H}_K}^2 : f \in \mathcal{H}_K, f(x_i) = y_i, \forall i=1,\dots,n \right\}.
    \end{align*}
    Since the infimum is taken over all interpolating functions, the norm of \textit{any} specific interpolating function $f_{construct}$ provides an upper bound:
    \begin{align*}
        y^T \mathbf{K}^{-1} y \leq \|f_{construct}\|_{\mathcal{H}_K}^2.
    \end{align*}
    
    \textbf{Construction of the Interpolant:}
    We construct $f_{construct}$ using localized bump functions. Let $\eta \in C_0^\infty(\mathbb{R}^d)$ be a fixed smooth cutoff function satisfying:
    \begin{itemize}
        \item $\text{supp}(\eta) \subset B(0, 1)$ (supported in the unit ball),
        \item $\eta(0) = 1$,
        \item $0 \leq \eta(x) \leq 1$.
    \end{itemize}
    For each data point $x_i$, define the scaled bump function $\psi_i$ as:
    \begin{align*}
        \psi_i(x) := \eta\left( \frac{x - x_i}{q_X} \right).
    \end{align*}
    By the definition of the separation distance $q_X$, the supports of $\psi_i$ are mutually disjoint:
    \begin{align*}
        \text{supp}(\psi_i) \cap \text{supp}(\psi_j) = \emptyset, \quad \forall i \neq j.
    \end{align*}
    This implies $\psi_i(x_j) = \delta_{ij}$. We now define the interpolating function as:
    \begin{align*}
        f_{construct}(x) := \sum_{i=1}^n y_i \psi_i(x).
    \end{align*}
    Clearly, $f_{construct}(x_j) = \sum_i y_i \delta_{ij} = y_j$, satisfying the interpolation condition.
    
    \textbf{Norm Estimation via Scaling:}
    The norm in $\mathcal{H}_K$ is equivalent to the Sobolev norm $\|\cdot\|_{H^s}$ with $s = \frac{d+1}{2}$. Since the supports of $\psi_i$ are disjoint, the $L^2$-mass and the dominant derivative energy (for integer $s$ or locally dominant fractional parts) add up. Thus, we have the bound:
    \begin{align*}
        \|f_{construct}\|_{H^s}^2 = \left\| \sum_{i=1}^n y_i \psi_i \right\|_{H^s}^2 \le C \sum_{i=1}^n y_i^2 \|\psi_i\|_{H^s}^2.
    \end{align*}
    Now we compute the Sobolev norm of a single scaled bump function $\psi_i$. Using the change of variables $z = (x-x_i)/q_X$, we have $dx = q_X^d dz$. The $s$-th order derivative operator $\nabla^s$ introduces a scaling factor of $q_X^{-s}$.
    \begin{align*}
        \|\psi_i\|_{H^s}^2 &\asymp \int_{\mathbb{R}^d} |\nabla^s \psi_i(x)|^2 dx \nonumber \\
        &= \int_{\mathbb{R}^d} \left| q_X^{-s} (\nabla^s \eta)\left(\frac{x-x_i}{q_X}\right) \right|^2 dx \nonumber \\
        &= q_X^{-2s} \cdot q_X^d \int_{\mathbb{R}^d} |\nabla^s \eta(z)|^2 dz \nonumber \\
        &= C_\eta \cdot q_X^{d-2s}.
    \end{align*}
    Substituting $s = \frac{d+1}{2}$, the exponent becomes:
    \begin{align*}
        d - 2s = d - (d+1) = -1.
    \end{align*}
    Thus, $\|\psi_i\|_{H^s}^2 \le C' q_X^{-1}$ for all $i$.
    
    \textbf{Conclusion:}
    Combining the results and using the fact that $y \in S^{n-1}$ (i.e., $\sum_{i=1}^n y_i^2 = 1$):
    \begin{align*}
        y^T \mathbf{K}^{-1} y &\leq \|f_{construct}\|_{H^s}^2 \nonumber \\
        &\leq C \sum_{i=1}^n y_i^2 \cdot (C' q_X^{-1}) \nonumber \\
        &= C \cdot C' \cdot q_X^{-1} \underbrace{\sum_{i=1}^n y_i^2}_{=1} \nonumber \\
        &= A q_X^{-1}.
    \end{align*}
    This completes the proof of the lower bound for the eigenvalue (or equivalently, the upper bound for the quadratic form).
    
    Now we prove the upper bound. Since $K^{NT}$ is equivalent to the exponential kernel, we only have to prove it for the exponential kernel. Using Cauchy's interlace theorem, taking the $i,j$ for which $x_i,x_j$ are the closest and taking the corresponding order 2 principal minor of the kernel matrix of the exponential kernel, its smallest eigenvalue is $1-e^{-q_X}\leq q_X$, thus we obtain the result.
\end{proof}

\begin{lemma}[Exponential Moment Bound for Minimum Separation]
    Let $X_1, \dots, X_n$ be independent and identically distributed random variables with probability density lower and upper bounded on the unit hypercube $[0,1]^d$. Let $q_X = \min_{1 \le i < j \le n} \|X_i - X_j\|_2$ denote the minimum pairwise distance. For any parameter $t > 0$, the following bound holds:
    \begin{align*}
        \mathbb{E}\left[ e^{-t q_X} \right] \le \frac{V_d \Gamma(d+1)}{2} \cdot \frac{n^2}{t^d},
    \end{align*}
    where $V_d = \frac{\pi^{d/2}}{\Gamma(d/2+1)}$ is the volume of the $d$-dimensional unit ball.
    
    Consequently, if $t = t(n)$ is chosen such that $t(n) \ge C n^{6/d}$ for a sufficiently large constant $C$, we obtain the convergence rate:
    \begin{align*}
        \mathbb{E}\left[ e^{-t(n) q_X} \right] = O(n^{-4}).
    \end{align*}
\end{lemma}

\begin{proof}
    We express the expectation of the non-negative random variable $Z = e^{-t q_X}$ using the integral of its tail probability. Since $q_X \ge 0$, $Z \in (0, 1]$. We have:
    \begin{align*}
        \mathbb{E}[e^{-t q_X}] &= \int_0^1 \mathbb{P}(e^{-t q_X} > y) \, dy.
    \end{align*}
    Apply the change of variable $y = e^{-tu}$. Then $dy = -t e^{-tu} du$, and the integration limits change from $(0, 1]$ to $(\infty, 0]$. This yields:
    \begin{align*}
        \mathbb{E}[e^{-t q_X}] &= \int_\infty^0 \mathbb{P}(e^{-t q_X} > e^{-tu}) (-t e^{-tu}) \, du \nonumber \\
        &= \int_0^\infty t e^{-tu} \mathbb{P}(q_X < u) \, du. 
    \end{align*}
    
    Next, we derive an upper bound for the cumulative distribution function $F(u) = \mathbb{P}(q_X < u)$. By definition, the event $\{q_X < u\}$ is the union of events where any pair of points is closer than $u$. Applying the Union Bound (Boole's inequality):
    \begin{align*}
        \mathbb{P}(q_X < u) &= \mathbb{P}\left( \bigcup_{1 \le i < j \le n} \{ \|X_i - X_j\| < u \} \right) \\
        &\le \sum_{1 \le i < j \le n} \mathbb{P}(\|X_i - X_j\| < u).
    \end{align*}
    Since the samples are i.i.d., there are $\binom{n}{2}$ identical terms. For any $u > 0$, the probability that two points in $[0,1]^d$ are within distance $u$ is bounded by the volume of a ball of radius $u$, denoted as $V_d u^d$. Note that boundary effects in the unit hypercube only reduce the intersection volume, so the full ball volume remains a strictly valid upper bound:
    \begin{align*}
        \mathbb{P}(q_X < u) \le \frac{n(n-1)}{2} V_d u^d \le \frac{n^2 V_d}{2} u^d. 
    \end{align*}
    Substituting back into the inequality:
    \begin{align*}
        \mathbb{E}[e^{-t q_X}] &\le \int_0^\infty t e^{-tu} \left( \frac{n^2 V_d}{2} u^d \right) du \\
        &= \frac{n^2 V_d}{2} t \int_0^\infty e^{-tu} u^d \, du.
    \end{align*}
    Let $x = tu$, so $u = x/t$ and $du = dx/t$. The integral becomes the Gamma function definition:
    \begin{align*}
        \int_0^\infty e^{-tu} u^d \, du &= \int_0^\infty e^{-x} \left(\frac{x}{t}\right)^d \frac{dx}{t} \\
        &= \frac{1}{t^{d+1}} \int_0^\infty e^{-x} x^d \, dx \\
        &= \frac{\Gamma(d+1)}{t^{d+1}}.
    \end{align*}
    Combining terms, we obtain the general bound:
    \begin{align*}
        \mathbb{E}[e^{-t q_X}] &\le \frac{n^2 V_d}{2} t \cdot \frac{\Gamma(d+1)}{t^{d+1}} \\
        &= \frac{V_d \Gamma(d+1)}{2} \frac{n^2}{t^d}.
    \end{align*}
    Finally, to satisfy the condition that the expectation is $O(n^{-4})$, we require:
    \begin{align*}
        \frac{n^2}{t(n)^d} \lesssim n^{-4} \implies t(n)^d \gtrsim n^6 \implies t(n) \gtrsim n^{6/d}.
    \end{align*}
    Thus, choosing $t(n) \ge C n^{6/d}$ for an appropriate constant suffices.
\end{proof}

\begin{proof}[Proof of Theorem 5.3]
    We deduce a supnorm bound of $\hat{f}_t$ to $\hat{f}_{\infty}$:
    \begin{align*}
        |\hat{f}_t(x)-\hat{f}_{\infty}(x)| &= |k(x,X)e^{-\frac{t}{n}k(X,X)}k(X,X)^{-1}y| \\
        &\leq |k(x,X)||y|e^{-\frac{t}{n} \frac{1}{C}q_X}.
    \end{align*}
    By Lemma 5.2.
    Thus, taking expectations:
    \begin{align*}
        \mathbb{E}[\sup_{x}|\hat{f}_t(x)-\hat{f}_{\infty}(x)|^2]
        &\leq \mathbb{E}[\sup_{x}|k(x,X)|^2|y|^2e^{-\frac{2t}{n}\frac{1}{C}q_X}] \\
        &\lesssim n^2 \mathbb{E}[e^{-\frac{2t}{nC}q_X}].
    \end{align*}
    Let $\lambda_n = \frac{2t}{nC}$. Utilizing the bound on the moment of the minimum separation distance, we have $\mathbb{E}[e^{-\lambda_n q_X}] \lesssim n^2 \lambda_n^{-d}$.
    Substituting $t=Cn^{6/d+1}$, the exponent coefficient becomes:
    \begin{equation*}
        \lambda_n = \frac{2Cn^{6/d+1}}{nC} = 2n^{6/d}.
    \end{equation*}
    Plugging this into the upper bound:
    \begin{align*}
        \mathbb{E}[\sup|\hat{f}_t-\hat{f}_{\infty}(x)|^2] 
        &\lesssim n^2 \cdot \left( \frac{n^2}{(2n^{6/d})^d} \right) \\
        &= n^2 \cdot \frac{n^2}{2^d n^6} \\
        &= \frac{1}{2^d} n^{-2}.
    \end{align*}
    Therefore, we conclude that $\mathbb{E}[\sup|\hat{f}_t-\hat{f}_{\infty}(x)|^2] \lesssim n^{-2}$.
    
    Now, for two functions $f$ and $g$, if $\sup|f-g|\leq a$ and $G_A^2(f)=b$:
    \begin{align*}
        2G_A^2(g)
        &\geq \int_{[0,1]^d}\sup_{x'\in B(x,r)}|f(x')-f(x)-2a|_+^2 \\
        &\geq \int_{[0,1]^d}\sup_{x'\in B(x,r)}\frac{1}{2}|f(x')-f(x)|^2-4a^2
        =2b-4a^2.
    \end{align*}
    Thus, 
    \begin{align*}
        R_A(\hat{f}_t, f^*) 
        &\geq \mathbb{E} [G_A^2(\hat{f}_t)] \notag \\
        &\geq \mathbb{E} \left[ G_A^2(\hat{f}_\infty, f^*) - 2 \sup_{x \in [0,1]^d} |\hat{f}_t(x) - \hat{f}_\infty(x)|^2 \right] \notag \\
        &\gtrsim \log(nr_n^d)-n^{-2}\gtrsim \log(nr_n^d). \notag
    \end{align*}
    
\end{proof}

\begin{proof}[Proof of Theorem 5.5]
    Once we establish the smallest eigenvalue of the kernel matrix $k_n$, the argument from the proof of Theorem 5.3 applies directly.
    
    \begin{lemma}
        Let $A$ and $B$ be two positive definite matrices. Then the smallest eigenvalue of $A+B$ is strictly greater than the smallest eigenvalue of $A$.
    \end{lemma}
    
    \begin{proof}
        By the Rayleigh quotient characterization, the smallest eigenvalue of a symmetric (or Hermitian) matrix $M$ is $\lambda_{\min}(M) = \min_{\lVert x \rVert = 1} x^T M x$. 
        Let $v$ be a normalized eigenvector corresponding to $\lambda_{\min}(A+B)$. We have:
        \[
        \lambda_{\min}(A+B) = v^T (A+B) v = v^T A v + v^T B v.
        \]
        By definition, $v^T A v \ge \lambda_{\min}(A)$. Furthermore, since $B$ is positive definite and $v \neq 0$, we strictly have $v^T B v > 0$. Therefore:
        \[
        \lambda_{\min}(A+B) \ge \lambda_{\min}(A) + v^T B v > \lambda_{\min}(A).
        \]
        This completes the proof of the lemma.
    \end{proof}
    
    Applying this lemma to the two components of $k_n$—the exponential scaling term and $K^{NT}$—completes the proof.
\end{proof}
\end{document}